\documentclass[preprint,12pt]{elsarticle}

\usepackage[utf8]{inputenc}
\usepackage{amsthm}
\usepackage{booktabs}
\usepackage{float}
\usepackage[utf8]{inputenc}
\usepackage{dirtytalk}
\usepackage[dvipsnames]{xcolor}
\usepackage{algorithm}
\usepackage{svg}
\usepackage[noend]{algpseudocode}
\usepackage{subfig}
\usepackage{enumitem}
\usepackage{tikz}
\usepackage{multirow}
\usepackage[export]{adjustbox}
\usepackage{amsmath}
\usepackage{amssymb}
\usepackage{float}
\usepackage{svg}
\usepackage{enumitem}
\usepackage{makecell}
\usepackage{graphicx}

\theoremstyle{definition}
\newtheorem{definition}{Definition}[section]
\newtheorem{lemma}{Lemma}[section]

\newtheorem{proposition}{Proposition}[section]

\journal{arxiv.org}

\begin{document}

\begin{frontmatter}



\title{A Novel Machine Learning Approach to Data Inconsistency with respect to a Fuzzy Relation}


\author{ Marko Palangeti\' c$^{a}$, Chris Cornelis$^a$,  Salvatore Greco$^{b,c}$, Roman S\l{}owi\' nski$^{d,e}$\\}

\address{$^a$Department of Applied Mathematics, Computer Science and Statistics, \\ Ghent University, Ghent, Belgium, \{marko.palangetic, chris.cornelis\}@ugent.be  \\
$^b$Department of Economics and Business, University of Catania, Catania, Italy,\\  salgreco@unict.it\\
$^c$Portsmouth Business School, Centre of Operations Research and Logistics (CORL), \\
University of Portsmouth, Portsmouth, United Kingdom \\
$^d$Institute of Computing Science,  Pozna\'n University of Technology, Pozna\'n, Poland,\\ roman.slowinski@cs.put.poznan.pl \\
$^e$Systems Research Institute, Polish Academy of Sciences, Warsaw, Poland \\}

\begin{abstract}
Inconsistency in prediction problems occurs when instances that relate in a certain way on condition attributes, do not follow the same relation on the decision attribute. For example, in ordinal classification with monotonicity constraints, it occurs when an instance dominating another instance on condition attributes has been assigned to a worse decision class. 
It typically appears as a result of perturbation in data caused by incomplete knowledge (missing attributes) or by random effects that occur during data generation (instability in the assessment of decision attribute values). Inconsistencies with respect to a crisp preorder relation (expressing either dominance or indiscernibility between instances) can be handled using symbolic approaches like rough set theory and by using statistical/machine learning approaches that involve optimization methods. Fuzzy rough sets can also be seen as a symbolic approach to inconsistency handling with respect to a fuzzy relation. In this article, we introduce a new machine learning method for inconsistency handling with respect to a fuzzy preorder relation. The novel approach is motivated by the existing machine learning approach used for crisp relations. We provide statistical foundations for it and develop optimization procedures that can be used to eliminate inconsistencies. The article also proves important properties and contains didactic examples of those procedures.

\end{abstract}


\end{frontmatter}


\section{Introduction}
Ordinal classification (also called ordinal regression) problems constitute a very important part of machine learning and statistical analysis \cite{gutierrez2015ordinal}. In ordinal classification, the goal is to predict for a certain instance $u$ from set $U$, one of $K$ different ordinal class labels $y \in \{1, \dots, K \}$. Usually, $u$ is characterised by its values for a given set of condition attributes, while $y$ is called a decision attribute. Ordinal classification problems exploit the existing ordering on the decision attribute. 
In some cases, an ordering also exists on the condition attributes. One way to incorporate that knowledge is through so-called monotonicity constraints. For a given preorder (dominance) relation on the set of instances $U$ based on the condition attributes, the monotonicity constraints can be formulated as follows: if instance $u_1$ dominates $u_2$ w.r.t. the given dominance relation on the condition attributes, then $u_1$ should be assigned to the same or to a better decision class than $u_2$. In such a case, we say that $u_1$ is consistent with $u_2$. Obviously, consistency is a symmetric relation and instances that are incomparable w.r.t.\ condition attributes are consistent by default.

Ordinal classification problems that include monotonicity constraints are called monotone classification problems. They arise in many areas, such as medical diagnosis \cite{chandrasekaran2005isotonic}, survey data \cite{cao2004instance}, bankruptcy risk estimation \cite{greco1998new}, house pricing \cite{potharst2002classification} and others. A comprehensive survey of monotone classification methods is given in \cite{cano2019monotonic}.

In practice, not all pairs of instances satisfy the monotonicity constraints (are not consistent) due to some imperfectness of ordinal classification data, like missing attributes or instability of the assessment of decision attribute values at the stage of data generation. Hence, some initial preprocessing of data is often performed in order to enforce the monotonicity constraints. One of the main methods to handle inconsistencies in data is the rough set approach \cite{pawlak1982rough}. For a given decision class in a classification problem, the approach outputs lower and upper approximations of that class. The lower approximation contains instances from the decision class that are consistent with all other instances, while the upper approximation contains instances that relate with instances from the decision class. The original rough set approach handles inconsistencies w.r.t.\ an equivalence (indiscernibility) relation and not w.r.t.\ a dominance relation (as is defined in the monotonicity constaints). To make it applicable for monotone classification problems, Greco et al.\ \cite{greco2001rough} extended the original rough set theory with their Dominance-based Rough Set Approach (DRSA) where the new approach replaces the indiscernibility relation by a dominance relation. After the introduction of DRSA, the original rough set theory is usually referred to as Indiscernibility-based Rough Set Approach (IRSA). Recently, the two were integrated into the Preorder-based Rough Set Approach (PRSA) \cite{palangetic2021granular}.


A more comprehensive analysis of monotone classification from the statistical learning point of view was given by Kotłowski and Słowiński \cite{kotlowski2008statistical}. They provided statistical foundations of the monotonicity constraints and developed a machine learning method to incorporate them into data analysis. This method removes inconsistencies in data (``monotonizes" them) in result of an optimization procedure that minimizes the cost of label changes in the decision attribute. It produces a new set of labels called a \textit{monotone approximation}. This approach also generalizes standard rough sets and provides another probabilistic view of them. The approach found its application in the same areas as DRSA \cite{kotlowski2008stochastic}, as well as in the development of rule induction and ensemble rules methods \cite{dembczynski2008ensemble}. In the remainder of the paper, we refer to this method as \textit{KS approach}.


On the other hand, fuzzy logic and fuzzy set theory \cite{zadeh1965} study the gradual truth of logical statements, and are used extensively in modeling imprecise and vague information. One of the ways to utilize fuzzy logic in data analysis is through fuzzy relations that are able to model relationships between numerical vectors. Namely, the usual crisp relations may distinguish only between two extreme cases: either instances relate or do not. Fuzzy relations, on the other hand, can express a degree in which two instances relate on a scale between 0 and 1. This is suitable to model similarity between numerical vectors or other structures (graphs, strings, DNA chains ...). The  integration of fuzzy logic and IRSA was initially proposed by Dubois and Prade \cite{dubois1990rough}, allowing to approximate fuzzy sets using a fuzzy indiscernibility relation. A similar extension of DRSA to fuzzy set theory was proposed by Greco et al. \cite{greco2000fuzzy}.


This article is motivated by the KS approach introduced above in the sense that we generalize the monotonicity constraints using fuzzy relations while the ordinal classes are replaced with fuzzy membership degrees. Instead of a crisp preorder relation (or an equivalence relation if it is symmetric), we will now consider a fuzzy $T$-preorder relation to model the relationship between different instances on the condition attributes, where $T$ refers to a given $t$-norm that models conjunction in fuzzy logic. The $T$-preorder relations also include $T$-equivalence relations that can measure (symmetric) similarity between numerical vectors. Moreover, the new approach requires that the decision attribute is a fuzzy set, i.e., it has to take values from interval $[0,1]$. Hence, it is appropriate for problems where the decision attribute can be modeled using values from this interval; concretely, for binary classification and regression problems.


Just like the KS approach \cite{kotlowski2008statistical}, our proposal is also interesting from the granular computing point of view. Granular computing is a paradigm which involves a partition of information into meaningful groups, classes or clusters called granules \cite{zadeh1979fuzzy, zadeh1997toward, bargiela2006roots}, and which has been applied to diverse models in data analysis. For example, in \cite{lin1998granular} and \cite{yao1999granular}, granular computing using neighborhood systems for the interpretation of granules was studied, while in \cite{pawlak1998granularity, pawlak2002granularity, polkowski1999towards, skowron2005approximation}, granular aspects of rough set theory were examined. 

In particular, the sets obtained with the Kotłowski-Słowiński approach, as well as with the novel approach, possess the property of granular representation: they can be represented as unions of meaningful granules \cite{palangetic2021granular,yao1999rough,greco2010dominance}. Such sets are called granularly representable sets \cite{palangetic2021granular}. Due to the granular properties of our new approach, we call 
its result a \textit{granular approximation}.

Granularity of fuzzy rough sets has been already applied in rule induction \cite{zhao2009building}, and we expect that the granular approximations we propose can serve the same purpose. Since granular approximations are natural generalizations of fuzzy rough sets, they can be applied in methods like Fuzzy-Rough Nearest Neighbours (FRNN) for classification \cite{jensen2011fuzzy, sarkar2007fuzzy}, Fuzzy Rough Prototype Selection (FRPS) \cite{verbiest2013frps}, Fuzzy Rough Feature Selection (FRFS) \cite{cornelis2010attribute,qian2015fuzzy}, and so on.

The article is also an example of a successful integration of ideas and contributions of rough sets, fuzzy sets and machine learning. Handling inconsistency and granulation are main contributions of rough sets. The theory of fuzzy sets allows us to use fuzzy relations to model a non-binary interaction among instances. For example, fuzzy relations can be used in modeling similarities between instances represented by numerical vectors. At the end, including statistical/machine learning allows us to make data consistent, incurring the least possible cost (w.r.t. some loss function) using optimization methods.


The remainder of the paper is organized as follows. In Section \ref{sec:preliminaries}, we recall the required preliminaries about statistical learning theory, monotone approximations, fuzzy logic, and fuzzy rough sets. In Section \ref{sec:statistical_granular}, we develop the statistical foundations of granular approximations. Section \ref{sec:calculation} deals with optimization problems that output granular approximations, while their important properties with proofs are given in Section \ref{sec:properties}. Section \ref{sec:toy_examples} provides didactic examples for the approaches from Section \ref{sec:calculation}. Section \ref{sec:conclusion} contains our conclusion and outlines future work.

In Appendix A--D, we deal with the dual formulations of the optimization problems introduced in Section \ref{sec:calculation}. 
Using the duality theory, we obtain greedy algorithms for the optimization problems from Section \ref{sec:calculation} that allow us to prove Proposition \ref{prop:parameter_monotonicity}.

\section{Preliminaries}

\label{sec:preliminaries}

\subsection{Statistical learning for monotone classification}
\label{subsec:stat_learning}
A random variable $\mathcal{X}$ is a mapping from a probability space to a certain codomain $X$. If the codomain is a subset of the real numbers, $\mathcal{X}$ is usually characterized with a cumulative distribution function (CDF) defined as $F_{\mathcal{X}} = P(\mathcal{X} \leq x)$ for $x \in X$. A CDF is a non-decreasing and right-continuous function with codomain $[0,1]$. If the CDF is continuous then we say that $\mathcal{X}$ is continuous, while if the image of the CDF is a finite set, we say that $\mathcal{X}$ is discrete. Based on the CDF, a quantile function may be defined as follows: $Q_{\mathcal{X}}(p) = \inf\{y; F_{\mathcal{X}}(y) \geq p \}$ for $0 < p < 1$. In other words, if $p$ is in the image of $F_{\mathcal{X}}$, then $Q_{\mathcal{X}}(p)$ is the smallest value for which $P(\mathcal{X} \leq Q_{\mathcal{X}}(p)) = p$. The value $Q_{\mathcal{X}}(\frac{1}{2})$ is called the median of $\mathcal{X}$. The expected value of $\mathcal{X}$ can be expressed using the quantile function \cite{gilchrist2000statistical}:  
\begin{equation}
\label{eq:expectation_quantile}
E(\mathcal{X}) = \int_{0}^1 Q_{\mathcal{X}}(p) dp.
\end{equation}

We say that $\mathcal{X}_1$ \textit{stochastically dominates} $\mathcal{X}_2$ if $F_{\mathcal{X}_1}(x) \geq F_{\mathcal{X}_2}(x)$ for all $x \in X$.

\begin{proposition} \cite{shaked2007stochastic} For two random variables $\mathcal{X}_1$ and $\mathcal{X}_2$, it holds that 
$$
\forall x \in X, F_{\mathcal{X}_1}(x) \leq F_{\mathcal{X}_2}(x) \Leftrightarrow \forall p \in (0,1), Q_{\mathcal{X}_1}(p) \geq Q_{\mathcal{X}_2}(p).
$$
\end{proposition}
The above proposition states that the stochastic dominance can be characterized using quantile functions instead of CDFs.

We now examine the \textit{prediction problem}. Let $\mathcal{X}$ and $\mathcal{Y}$ be two random variables with codomains $X$ and $Y$ respectively. When making predictions, we examine a causal relationship between $\mathcal{X}$ and $\mathcal{Y}$, i.e., how does $\mathcal{X}$ influence $\mathcal{Y}$. Concretely, we are interested to find a function $h$ such that $h(\mathcal{X})$ is close to $\mathcal{Y}$, i.e., it predicts values of $\mathcal{Y}$ for given values of $\mathcal{X}$. Formally, let $L:Y \times Y \rightarrow \mathbb{R}^{+}$ be a loss function. A prediction problem consists in finding a function $h: X \rightarrow Y$ such that the risk 
$$
R(h) = E(L(\mathcal{Y}, h(\mathcal{X})))
$$ is minimized. The optimal $h$, denoted as $h^*$, is called the Bayes predictor. In practice, the random variables $\mathcal{X}$ and $\mathcal{Y}$ are unknown and we only have their realizations $x_1, \dots, x_n$ and $y_1, \dots, y_n$. Our goal is then to minimize the empirical risk: 
$$
\hat{R}(f) = \frac{1}{n}\sum_{i=1}^n L(y_i, h(x_i)).
$$
Minimization of the empirical risk is called \textit{learning} and it basically amounts to an estimation of the Bayes predictor. The causal relationship between $\mathcal{X}$ and $\mathcal{Y}$ may be represented by a family of random variables $\mathcal{Y}_{\mathcal{X} = x}$, which stands for variable $\mathcal{Y}$ conditioned on $\mathcal{X}=x$. Such a random variable, for a fixed $x$, may be described by its CDF:
$$
F_{\mathcal{Y}|\mathcal{X} = x}(y) = P(\mathcal{Y} \leq y| \mathcal{X} = x).
$$
Searching for an optimal prediction function $h$ in the learning process may be seen as an estimation of certain characteristics of the family of random variables $\mathcal{Y}_{\mathcal{X} = x}$. For example, when the loss function is mean squared error
\begin{equation}
L(y, \hat{y}) = (y - \hat{y})^2 \label{eq:mse},
\end{equation}
for $y, \hat{y} \in Y$ and $Y = \mathbb{R}$, then the Bayes predictor is $h^*(x) = E(\mathcal{Y}|\mathcal{X} = x)$, i.e., the conditional mean, while if the loss function is mean absolute error 
\begin{equation*}
L(y, \hat{y}) = |y - \hat{y}|,
\end{equation*}
then the Bayes predictor is $h^*(x) = Q_{\mathcal{Y}|\mathcal{X}=x}(\frac{1}{2})$, i.e., the conditional median \cite{berger2013statistical}. The previous examples show that a Bayes predictor is a characteristic of family $\mathcal{Y}_{\mathcal{X} = x}$ (conditional mean and median in the examples), which means that the learning process leads to an estimation of those characteristics.


Kotłowski and Słowiński \cite{kotlowski2008statistical} introduced a statistical framework for monotone classification. In this case, it is assumed that there is a preorder relation $\succeq_X$ on codomain $X$ of $\mathcal{X}$ while $Y$ consists of a finite number of totally ordered values that distinguish different ordinal classes. Denote these classes by $1, \dots, K$. The monotonicity constraint states that if $x \succeq x'$ then $x$ has to belong to at least the same class as $x'$. This is also called the Pareto principle in decision theory. Let $K_{-1} = \{ 1, \dots, K-1 \}$. In probabilistic terms, the monotonicity constraint says that $x \succeq x'$ implies
\begin{equation}
\label{eq:monotonically_constrained}
\begin{array}{ll}
&\forall k \in K_{-1}, \, P(\mathcal{Y} \leq k| \mathcal{X} = x) \leq P(\mathcal{Y} \leq k| \mathcal{X} = x') \\
\Leftrightarrow &  \forall k \in K_{-1}, \,  F_{\mathcal{Y}|\mathcal{X}=x}(k) \leq F_{\mathcal{Y}|\mathcal{X}=x'}(k) \\
\Leftrightarrow &  \forall p \in (0,1), \,   Q_{\mathcal{Y}|\mathcal{X}=x}(p) \geq Q_{\mathcal{Y}|\mathcal{X}=x'}(p).
\end{array}
\end{equation}
The previous expression means that the probability that $x$ will be assigned to a class at most $k$ is smaller or equal than that $x'$ will be assigned to the same class. A family $\mathcal{Y}_{\mathcal{X} = x}$ is \textit{monotonically constrained} if  (\ref{eq:monotonically_constrained}) is satisfied. A prediction function $h$ is called monotone if $x \succeq x' \implies h(x) \geq h(x')$. The goal of monotone classification is to find a proper monotone $h$ under the assumption that the family $\mathcal{Y}_{\mathcal{X} = x}$ is monotonically constrained. Since $h$, as the output of the learning process, should be as close as possible to the Bayes predictor $h^*$, we require that $h^*$ is also monotone. Given that the form of $h^*$ depends on the loss function, choosing a proper loss function is crucial for the learning process. A loss function for which the Bayes predictor is monotone is called a monotone loss function. Kotłowski and Słowiński \cite{kotlowski2008statistical} showed that both mean squared error and mean absolute error are monotone loss functions. They also examined a family of monotone loss functions called $p$-quantile loss defined as:
\begin{equation}
    L_p(y, \hat{y}) = (y - \hat{y}) (p - \mathbf{1}_{y-\hat{y} < 0}) = \begin{cases}
        p|y - \hat{y}| & \text{if   }  y-\hat{y} > 0, \\
        (1-p)|y - \hat{y}| & \text{otherwise}.
        \end{cases} \label{eq:quantile_loss}
\end{equation}
for $p \in [0,1]$, where $\mathbf{1}$ stands for the indicator function. The name $p$-quantile loss is used since the Bayes predictor for such loss function is the conditional $p$-quantile $h_{p}^*(x) = Q_{\mathcal{Y}|\mathcal{X} = x}(p)$. For $p=\frac{1}{2}$ we have that $L_{1/2}$ is equivalent to the mean absolute error. For the $p$-quantile loss, we have the following important result proved in \cite{kotlowski2008statistical}.
\begin{proposition}
\label{prop:quantile_scaling}
Let $s: \mathcal{Y} \rightarrow \mathbb{R}$ be an increasing function. Then the loss functions $L_p(y, \hat{y})$ and $L_p(s(y), s(\hat{y}))$ have the same Bayes predictor. 
\end{proposition}
Proposition \ref{prop:quantile_scaling} states that a different scaling of ordinal classes does not change the Bayes predictor, only the order matters. 

\subsection{Monotone approximation}
In order to incorporate monotonicity constraints into the learning process, KS approach use an optimization procedure to ``monotonize" data by eliminating inconsistencies. Let $\Bar{y}_i, i = 1, \dots, n$, be the observed ordinal labels which do not necessarily satisfy monotonicity constraints due to possible inconsistency, and let $\hat{y}_i, i = 1, \dots, n$, be the values that we want to \textit{learn} and which satisfy the constraints. Then, for a given monotone loss function $L$, the optimization problem can be formulated as 

\begin{equation}
\label{eq:monotone_approximation}
\begin{aligned}
&\text{minimize}  && \displaystyle\sum_{ i = 1}^{n}  L(\Bar{y}_i, \hat{y}_i) &&\\
&\text{subject to}    && x_i \succeq x_j \implies \hat{y}_i \geq \hat{y}_j,    && i,j  = 1, \dots, n \\
  &              && \hat{y}_i \in \{0, \dots, K\}, && i = 1, \dots, n
\end{aligned}
\end{equation}
In other words, one wants to calculate new labels that are as close as possible to the original ones w.r.t. loss function $L$ and which satisfy the monotonicity constraints. The obtained labels are called a \textit{monotone approximation} of the original ones. The same authors showed that when $L$ is monotone, then problem (\ref{eq:monotone_approximation}) can be solved using linear programming. Moreover, the solutions of the linear optimization problem will always be integers due to the unimodularity of the constraint matrix \cite{papadimitriou1998combinatorial}. 

\subsection{Fuzzy logic connectives}
\label{subsec:fuzzy_logic} 
The definitions and terminology in this subsection are based on \cite{klement2013triangular}. Recall that a {\em $t$-norm} $T:[0,1]^2 \rightarrow [0,1]$ is a binary operator which is commutative, associative, non-decreasing in both arguments, and for which it holds that $ \forall x \in [0,1] ,\, T(x,1) = x$. Since a $t$-norm is associative, we may extend it unambiguously to a $[0,1]^n \rightarrow [0,1]$ mapping for any $n > 2$. Some commonly used $t$-norms are listed in Table \ref{table:tnorms}.

We say that $x \in [0,1]$ is a nilpotent element of a $t$-norm $T$ if there exists a natural number $n$ such that
$$
T(\underbrace{x, \ldots, x}_{\mbox{$n$ times}}) = 0.
$$ A $t$-norm is called nilpotent if it is continuous and every $x \in (0,1)$ is a nilpotent element. For example, $T_L$ from Table \ref{table:tnorms} is nilpotent while the others are not.
A $t$-norm is strict if it is continuous and strictly increasing in both arguments. $T_P$ from Table \ref{table:tnorms} is strict while the others are not.

We call a $t$-norm Archimedean if 
$$
(\forall (x,y) \in (0, 1)^2)(\exists n \ge 2)(T(\underbrace{x, \ldots, x}_{\mbox{$n$ times}}) < y).
$$
$T_P$, $T_L$ and $T_D$ from Table \ref{table:tnorms} are Archimedean, while $T_M$ and $T_{nM}$ are not.

\begin{table}[H]
\begin{adjustbox}{width=\columnwidth,center}
    \begin{tabular}{c|rcl|rcl}
    Name & Definition & & & R-implicator\\
    \hline
Minimum &         $T_M(x,y)$ &=& $\min(x,y)$  & $I_{T_M}(x,y)$ &=& $\left\{\begin{array}{cc}
       1 & \mbox{if $x \le y$}  \\
            y & \mbox{otherwise} 
        \end{array}
        \right.$    \\
Product        & $T_P(x,y)$ &=& $xy$ & $I_{T_P}(x,y)$ &=& $\left\{\begin{array}{cc}
       1 & \mbox{if $x \le y$}  \\
            \frac{y}{x} & \mbox{otherwise} 
        \end{array}
        \right.$ \\
{\L}ukasiewicz    &    $T_L(x,y)$ &=& $\max(0,x+y-1)$ &    $I_{T_L}(x,y)$ &=& $\min(1,1-x+y)$ \\
Drastic     &   $T_D(x,y)$ &=& $\left\{\begin{array}{cc}
               \min(x,y) & \mbox{if $\max(x,y) = 1$}  \\
            0 & \mbox{otherwise} 
        \end{array}
        \right.$  &   $I_{T_D}(x,y)$ &=& $\left\{\begin{array}{cc}
           y & \mbox{if $x = 1$}  \\
            1 & \mbox{otherwise} 
        \end{array}
        \right.$\\
\makecell{Nilpotent \\ minimum} & $T_{nM}(x,y)$ &=& $\left\{\begin{array}{cc}
            \min(x,y) & \mbox{if $x + y > 1$}  \\
            0 & \mbox{otherwise} 
        \end{array}
        \right.$ & $I_{T_{nM}}(x,y)$ &=& $\left\{\begin{array}{cc}
            1 & \mbox{if $x \le y$}  \\
            \max(1-x,y) & \mbox{otherwise} 
        \end{array}
        \right.$\\
        \hline
    \end{tabular}
    \end{adjustbox}
    \vspace{8pt}
    \caption{Some common $t$-norms and their R-implicators}
    \label{table:tnorms}
\end{table}

A $t$-norm is a continuous Archimedean $t$-norm if and only if it is either strict or nilpotent.

We say that two $t$-norms $T_1$ and $T_2$ are isomorphic if there exists a bijection $\varphi:[0,1] \rightarrow [0,1]$ such that $T_1 = \varphi^{-1}(T_2(\varphi(x), \varphi(y)))$.

\begin{proposition}
\label{prop:t_norm_ispomorphic}
A strict $t$-norm is isomorphic to $T_P$ while a nilpotent $t$-norm is isomorphic to $T_L$.
\end{proposition}
We denote with
\begin{equation}
    T_{L, \varphi} = \varphi^{-1}(\max(\varphi(x) + \varphi(y) -1, 0)) \label{eq:nilpotent}
\end{equation}
 a family of nilpotent $t$-norms, i.e., $t$-norms that are isomorphic to $T_L$ with bijection $\varphi$ and we denote with
\begin{equation}
T_{P, \varphi} = \varphi^{-1}(\varphi(x) \varphi(y)) \label{eq:strict}
\end{equation}
 a family of strict $t$-norms, i.e., $t$-norms that are isomorphic to $T_L$ with bijection $\varphi$.

We say that a $t$-norm is $D$-convex if its partial mappings are convex functions. 
$T_P$ and $T_L$ from Table \ref{table:tnorms} are $D$-convex, while the others are not. 
More details on the characterization of $D$-convex $t$-norms can be found in \cite{palangetic2021granular}.

An {\em implicator}  (or {\em fuzzy implication}) $I : [0,1]^2 \rightarrow [0,1]$ is a binary operator which is non-increasing in the first component, non-decreasing in the second one and for which it holds that $I(1,0) = 0$ and $I(0,0) = I(0,1) = I(1,1) = 1$. 

The residuation property holds for a $t$-norm $T$ and an implicator $I$ if 
$$
T(x,y) \leq z \Leftrightarrow x \leq I(y,z).
$$ 
It is satisfied if and only if $T$ is left-continuous and $I$ is defined as the residual implicator (R-implicator) of $T$, that is
$$I_T(x,y) = \sup \{\lambda  \in [0,1]; T(x,\lambda) \leq y\}.$$

The very right column of Table \ref{table:tnorms} shows the residual implicators of the corresponding $t$-norms. Note that all of them, except $I_{T_D}$, satisfy the residuation property. 
Implicators that satisfy the residuation principle have the ordering property
\begin{equation}
\label{eq:ordering_property}
    x \leq y \Leftrightarrow I(x,y) = 1.
\end{equation}
Given a $[0,1] \rightarrow [0,1]$ bijection $\varphi$, the residual implicators of nilpotent and strict $t$-norms $T_{L, \varphi}$ and $T_{P, \varphi}$ will be denoted by $I_{L, \varphi}$ and $I_{P, \varphi}$.

A {\em negator}  (or { \em fuzzy negation}) $N : [0,1] \rightarrow [0,1]$ is a unary non-increasing operator for which it holds that $N(0) = 1$ and $N(1) = 0$. A negator is involutive if $N(N(x)) = x$ for all $x \in [0,1]$. The standard negator is defined as $N_s(x) = 1-x$.

For a left continuous $t$-norm $T$ and its R-implicator $I$, we define the negator induced by $I$ as $N(x) = I(x, 0)$.
We will call a triplet $(T, I, N)$ obtained as previously explained a residual triplet.
If a $t$-norm from a residual triplet is continuous and Archimedean, then the negator of the triplet is involutive if and only if the $t$-norm is nilpotent. In such case, the negator has the form
$$
N_{\varphi}(x) = \varphi^{-1}(1- \varphi(x)).
$$

\begin{proposition}
\label{prop:residual_property}
For a residual triplet, the following holds:
$$
I(T(x,y), z) = I(x,I(y, z)).
$$
As a consequence, when $z=0$,
$$
N(T(x, y)) = I(x, N(y)).
$$
\end{proposition}

The standard negator is obtained when, for example, the $t$-norm is the Łukasiewicz one. In general, a $t$-norm for which the induced negator of its R-implicator is involutive is called an IMTL $t$-norm. We will call a residual triplet $(T, I, N)$, that is generated with an IMTL $t$-norm, an IMTL triplet. 

\subsection{Fuzzy sets and fuzzy relations}

Given a non-empty set $U$, a fuzzy set $A$ in $U$ is an ordered pair $(U, m_A)$, where $m_A:U \rightarrow [0,1]$ is a membership function that indicates how much an element from $U$ is contained in $A$. Instead of $m_A(u)$, the membership degree is often written as $A(u)$. If the image of $m_A$ is $\{0,1\}$ then we obtain a crisp or classical set. For a negator $N$, the fuzzy complement $coA$ is defined as $coA(u) = N(A(u))$ for $u\in U$. If $A$ is crisp then $coA$ reduces to the standard complement. For $\alpha \in (0,1]$, the $\alpha$-level set of fuzzy set $A$ is a crisp set defined as $A_{\alpha} = \{u \in U; A(u) \geq \alpha \}$.

A fuzzy relation $\widetilde{R}$ on $U$ is a fuzzy set on $U \times U$, i.e., a mapping $\widetilde{R}: U \times U \rightarrow [0,1]$ which indicates how much two elements from $U$ are related. Some relevant properties of fuzzy relations include:
\begin{itemize}
    \item $\widetilde{R}$ is reflexive if $\forall u \in U,\ \widetilde{R}(u,u) = 1$.
    \item $\widetilde{R}$ is symmetric if $\forall u,v \in U,\ \widetilde{R}(u,v) = \widetilde{R}(v,u)$.
    \item $\widetilde{R}$ is $T$-transitive w.r.t. $t$-norm $T$ 
    if $\forall u,v,w \in U$ it holds that  \\
    $T(\widetilde{R}(u,v),\widetilde{R}(v,w)) \leq \widetilde{R}(u,w)$.
\end{itemize}{}
A reflexive and $T$-transitive fuzzy relation is called a $T$-preorder relation while a symmetric $T$-preorder relation is called a $T$-equivalence.

\subsection{Fuzzy rough and granular approximations}
Let $U$ be a set of instances, $A$ a fuzzy set on $U$ and let $\widetilde{R}$ be a $T$-preorder relation on $U$. The fuzzy PRSA lower and upper approximations of $A$ are fuzzy sets for which the membership function is defined as:
\begin{equation}
\label{eq:fuzzy rough approx}
\begin{aligned}
\underline{\text{apr}}_{\widetilde{R}}^{\min, I}(A)(u) &= \min \{I(\widetilde{R}(v,u), A(v)); v \in U \} \\
\overline{\text{apr}}_{\widetilde{R}}^{\max, T}(A)(u) &= \max \{T(\widetilde{R}(u,v), A(v)); v \in U\},
\end{aligned}
\end{equation}
for $u \in U$. The approximations have some important properties \cite{palangetic2021fuzzy}:
\begin{itemize}
    \item (\textbf{inclusion}) $\underline{\text{apr}}_{\widetilde{R}}^{\min, I}(A) \subseteq A \subseteq \overline{\text{apr}}_{\widetilde{R}}^{\max, T}(A)$.
    \item (\textbf{duality}) $co\underline{\text{apr}}_{\widetilde{R}}^{\min, I}(A) =  \overline{\text{apr}}_{\widetilde{R}}^{\max, T}(coA)$, $ co \overline{\text{apr}}_{\widetilde{R}}^{\max, T}(A) = \underline{\text{apr}}_{\widetilde{R}}^{\min, I}(co A)$ when an IMTL triplet is used. 
    \item (\textbf{object monotonicity}) $ \widetilde{R}(u,v) \leq I(\underline{\text{apr}}_{\widetilde{R}}^{\min, I} (A) (v), \underline{\text{apr}}_{\widetilde{R}}^{\min, I} (A) (u))$, \\
         $ \widetilde{R}(u,v) \leq I(\overline{\text{apr}}_{\widetilde{R}}^{\max, T} (A) (v), \overline{\text{apr}}_{\widetilde{R}}^{\max, T} (A) (u))$.
\end{itemize}
Given $\lambda \in [0,1]$, $T$-preorder $\widetilde{R}$, $t$-norm $T$ and $u \in U$, a fuzzy granule is defined as a parametric fuzzy set 
\begin{align}
\label{eq:granule}
    \widetilde{R}^+_\lambda (u) = \{(v, T(\widetilde{R}(u,v), \lambda ); v \in U\}.
\end{align}
A fuzzy set $A$ in $U$ is granularly representable w.r.t.\ $\widetilde{R}$ and $T$ if 
$$
A = \bigcup \{ \widetilde{R}^+_{A(u)}(u); u \in U\},
$$
where the union of fuzzy sets is defined with the $\max$ operator.
\begin{proposition}
\label{prop:l_and_s_gp} \cite{palangetic2021granular}
It holds that $\underline{\text{apr}}_{\widetilde{R}}^{\min, I}(A)$ is the largest granularly representable set contained in $A$, while $\overline{\text{apr}}_{\widetilde{R}}^{\max, T}(A)$ is the smallest granularly representable set containing $A$.
\end{proposition}

\section{Statistical approach to granular representability}

\label{sec:statistical_granular}

\subsection{Ontic fuzzy sets and probabilistic uncertainty}

Fuzzy sets are often related to uncertainty modeling \cite{john2005modeling,ozkan2014uncertainty,davis1997modelling}. However, we should be very careful when mentioning that fuzzy sets are used to model uncertainty. First, two types of classical (crisp) sets have to be distinguished: conjunctive and disjunctive sets \cite{yager1987set}. A conjunctive set is a collection of items that represents a well known complex entity, i.e., it is a conjunction of its elements. For example, a time interval that describes a span of some activity. On the other hand, a disjunctive set describes incomplete information about an ill-known object. The object of interest is contained in the disjunctive set but we do not know which element it is, i.e., the set is a disjunction of its elements. For example, an event that occurred at an unknown moment in time is described with a time interval that represents our knowledge about the unknown event. Conjunctive sets are also known as ontic sets while disjunctive sets are called epistemic sets. Fuzzy sets are used to model gradual information which is not uncertain by itself. Fuzzy sets may be related to uncertainty only if the underlying universe, on which a fuzzy set is defined, is a disjunctive set. In that case, fuzzy sets make incomplete knowledge more expressive by allowing gradual information. Such fuzzy sets are usually known as epistemic fuzzy sets and form the basis of  possibility theory \cite{zadeh1978fuzzy}. In this article, we always use fuzzy sets defined over a conjunctive universe, i.e., ontic fuzzy sets, while we assume that the uncertainty in data is of probabilistic nature. An example of an ontic fuzzy set is a set of people that are ``tall", i.e., a fuzzy set whose universe is some set of humans and its membership function is a height measure of those humans. The height is an actual physical characteristic. In this case, no uncertainty or lack of knowledge exists.

\subsection{Granularly representable random fuzzy sets}

We assume that we observed a finite set of instances $U$ from the underlying universe, i.e., $U$ is a random sample. $U$ is described with condition and decision attributes where the decision attribute takes values in [0,1], which are interpreted as membership degrees to an unknown fuzzy set that we want to reconstruct using the observed values. From the perspective of statistical learning theory introduced in the previous section, condition attributes correspond to random variable $\mathcal{X}$ while the decision attribute corresponds to random variable $\mathcal{Y}$, which now takes values from interval $[0,1]$. The fuzzy set that we want to reconstruct contains uncertainties that are represented in a probabilistic way, i.e., we assume that the actual values are altered due to perturbation. Perturbation may be caused by the incomplete knowledge about data (missing attributes) or by random effects that occur during data generation. Such altered values are represented by a family of random variables $\{\mathcal{A}(u),\, u \in U \}$ which model our uncertainty about the ill-known membership degrees $\{ A(u), \, u \in U \}$. In other words, for each instance $u$, the ill-known membership degree $A(u)$ is represented with the random variable $\mathcal{A}(u)$ having codomain $[0,1]$. Family $\{\mathcal{A}(u),\, u \in U \}$ is a special case of a \textit{random fuzzy set} defined in \cite{PURI1986409} (the other name is \textit{fuzzy random variable}). Hence, we may refer to the family as \textit{random fuzzy set $\mathcal{A}$}.

The family $\{\mathcal{A}(u),\, u \in U \}$ corresponds to family $\mathcal{Y}_{\mathcal{X}=x}$ from the previous section. Therefore, we formulate the reconstruction of fuzzy set $A$ as the problem where for a given set of instances $U$ and its condition and decision attributes, we want to estimate characteristics of $\mathcal{A}(u)$ (like conditional mean, median and quantiles mentioned above) in order to describe the ill-known $A(u)$. Knowledge about condition attributes is represented using a $T$-preorder relation $\widetilde{R}$, i.e., for each pair $u,v \in U$ we are given the value $\widetilde{R}(u,v)$. We denote the observed decision values as $\Bar{A}(u)$ for $u \in U$.

In the first step, we will extend the monotonicity constraint (\ref{eq:monotonically_constrained}) for a $T$-preorder relation. 
We have the following simple result.
\begin{proposition}
\label{prop:granularity_object_monotonicity}
A fuzzy set $A$ in $U$ is granularly representable if and only if it satisfies the object monotonicity property, i.e.,
\begin{align*}
A = \bigcup \{ \widetilde{R}_{A(u)}(u); u \in U\}
&\Leftrightarrow  \forall u, v \in U, \, \widetilde{R}(v,u) \leq I(A(u), A(v)) \\
&\Leftrightarrow  \forall u, v \in U, \, T(\widetilde{R}(v,u) A(u)) \leq A(v) \\
\end{align*}
\end{proposition}
\begin{proof}

\begin{align*}
    A = \bigcup \{ \widetilde{R}_{A(u)}(u); u \in U\} &\Leftrightarrow \forall v \in U, \, A(v) = \max ( T(\widetilde{R}(v, u), A(u)); u \in U) \\ 
    &\Leftrightarrow \forall u, v \in U, \, A(v) \geq T(\widetilde{R}(v, u), A(u)) \\ 
    &\Leftrightarrow \forall u, v \in U, \, \widetilde{R}(v,u) \leq I(A(u), A(v)).
\end{align*}
The second equivalence holds from the observation that the maximum is reached for $u = v$ due to reflexivity of $\widetilde{R}$. The third equivalence holds from the residuation property. 
\end{proof}

Proposition \ref{prop:granularity_object_monotonicity} reveals that the granular representability is equivalent to the object monotonicity property. On the other hand, the object monotonicity property is the fuzzy extension of the crisp Pareto principle whose probabilistic form is given in (\ref{eq:monotonically_constrained}). That can be seen from the fact that inequality  $T(\widetilde{R}(v, u), A(u)) \leq A(v)$ is the fuzzy expression of the sentence: "If $v$ dominates $u$ and $u \in A$ then $v \in A$". We extend the principle provided in (\ref{eq:monotonically_constrained}) for the case when the preorder (dominance) relation "$\succeq$" is replaced with a fuzzy $T$-preorder relation $\widetilde{R}$ and when the decision attribute is expressed in the form of membership degrees rather than binary class membership. For the terminology, instead of ``monotonically constrained" we use the term ``granularly representable" introduced in \cite{palangetic2021granular}. The new term is more appropriate because a $T$-preorder relation is more general and may express different things, not only monotonicity (e.g., if the relation is a $T$-equivalence, then we are talking about similarity). In order to relate the granular representability and the family of random variables $\{\mathcal{A}(u), u \in U\}$, we introduce the following definition.
\begin{definition}
\label{def:granular_rep_random}
Random fuzzy set $\mathcal{A}$ is granularly representable if 
$$
\forall u,v \in U \, \text{and} \, \forall p \in [0,1], \, \widetilde{R}(u,v) \leq I(A_p(v), A_p(u)),
$$
where $A_p(u) = Q_{\mathcal{A}(u)}(p)$, i.e., $A_p$ is the conditional $p$-quantile of $\mathcal{A}$.
\end{definition}
Definition \ref{def:granular_rep_random} is an extension of the third equivalence in (\ref{eq:monotonically_constrained}). It states that $\mathcal{A}$ is granularly representable if all its $p$-quantiles $A_p$ ($p \in [0,1]$) are granularly representable as ordinal fuzzy sets. 

The next question is, if the random fuzzy set $\mathcal{A}$ is granularly representable, is its expected value $E\mathcal{A}$, defined as $E\mathcal{A} = \{E(\mathcal{A}(u)), u \in U\}$, also granularly representable? Before answering this question, we recall the well-known Jensen inequality \cite{rudin1987real}.
\begin{proposition}
\label{prop:Jensen_inequality}
Let $\mu$ be a probability measure on the set of reals, $g$ a $\mu$-measurable real function and $\phi$ a real convex function. It holds that
$$
\int \phi(g) d\mu \geq \phi \Bigg (\int g d\mu  \Bigg).
$$
Since the standard (Lebesgue) measure is equivalent to the probability measure on $[0,1]$ (measure value of interval $[0,1]$ is 1), the above inequality translates to
$$
\int_0 ^1 \phi(g(x)) dx  \geq \phi \bigg ( \int _0 ^1 g(x) dx \bigg ).
$$

\end{proposition}
Using Jensen's inequality, we obtain the following result.
\begin{proposition}
\label{prop:expectation_granularity}
Let $T$ be a $D$-convex $t$-norm and $I$ its R-implicator. Then $E\mathcal{A}$ is granularly representable as soon as $\mathcal{A}$ is.
\end{proposition}
\begin{proof}
For every $u, v \in U$, we need to prove that 
$$
T(\widetilde{R}(v, u), E\mathcal{A}(u)) \leq E\mathcal{A}(v).
$$
Using (\ref{eq:expectation_quantile}), we have that $\forall u \in U, \, E\mathcal{A}(u) = \int_0^1 A_p(u) dp $. It follows that
\begin{align*}
    T(\widetilde{R}(v, u), E\mathcal{A}(u)) &= T\left(\widetilde{R}(v, u), \int_0^1 A_p(u) dp\right)\\ &\leq \int_0^1 T(\widetilde{R}(v, u), A_p(u)) dp \\& \leq \int_0^1 A_p(v) dp = E\mathcal{A}(v).
\end{align*}
The first inequality follows from the fact that $T(c, \cdot)$ is a convex function for a constant $c$ and Jensen's inequality. The second inequality follows from the granularity of $A_p$.
\end{proof}

\section{Calculation of granular approximations}

\label{sec:calculation}

In this section, we discuss which properties of $\mathcal{A}$ can be estimated and how to do this in practice. In general, the observed fuzzy set $\Bar{A}$ is not granularly representable due to the presence of inconsistency, so our goal is to find a granularly representable set that is close to it by minimizing a certain loss function. For given loss function $L$, the general form of our optimization problem is 

\begin{equation}
\label{eq:general_optimization}
\begin{aligned}
&\text{minimize}  && \displaystyle\sum_{u \in U}  L(\Bar{A}(u), \hat{A}(u)) \\
&\text{subject to}    && T( \widetilde{R}(u,v), \hat{A}(v)) \leq \hat{A}(u), \quad   u, v \in U\\
  &              &&0 \leq \hat{A}(u) \leq 1, \quad u \in U,
\end{aligned}
\end{equation}
where $\{\hat{A}(u), u \in U \}$ is the unknown granularly representable set. We will call the result of optimization problem (\ref{eq:general_optimization}) the \textit{granular approximation} of fuzzy set $\{\Bar{A}(u), u \in U \}$. 

Optimization problem (\ref{eq:general_optimization}) is the main contribution of the article. It allows us to remove inconsistencies (obtain a granularly representable set) with the least cost of alteration of values (w.r.t. loss function $L$). The remainder of the section investigates specific cases for which problem (\ref{eq:general_optimization}) can be efficiently solved.

Under the assumption that $\mathcal{A}$ is granularly representable, it is desirable to use loss functions for which the Bayes predictor is granularly representable as well.
\begin{definition}
We say that a loss function $L$ is \textit{granular} with respect to a left-continuous $t$-norm $T$ and fuzzy relation $\widetilde{R}$ if its Bayes predictor is granularly representable under the assumption that the underlying family of random variables $\{\mathcal{A}(u), u \in U\}$ is granularly representable w.r.t.\ $T$ and $\widetilde{R}$.
\end{definition}

Note that with this definition, the $p$-quantile loss function (\ref{eq:quantile_loss}) is granular, since its Bayes predictor is the quantile fuzzy set $A_p$, which is granularly representable by definition. The mean squared error loss function (\ref{eq:mse}) is granular for D-convex $t$-norms since the Bayes predictor $E\mathcal{A}$ is granularly representable in this case by Proposition \ref{prop:expectation_granularity}. Hence, both loss functions introduced in Subsection \ref{subsec:stat_learning} are suitable for the calculation of granular approximations.

In problem (\ref{eq:general_optimization}), both objective function and constraints are not necessarily linear and may take different forms that depend on loss function $L$ and on the type of fuzzy logic connectives used. However, in case of the loss functions (\ref{eq:quantile_loss}) and (\ref{eq:mse}), and continuous Archimedean $t$-norms, the optimization problem can be efficiently solved. 

Indeed, consider $t$-norms $T_{L, \varphi}$ and $T_{P, \varphi}$ introduced in Eq.\ (\ref{eq:nilpotent}) and (\ref{eq:strict}). If $T_{L, \varphi}$ is used in (\ref{eq:general_optimization}), then the set of constraints that express granular representability can be simplified in the following way: for all $u,v \in U$,
\begin{eqnarray*}
&&\widetilde{R}(v,u) \leq I_{L,\varphi} (A(u), A(v)) \\
&\Leftrightarrow& T_{L,\varphi}\left(\widetilde{R}(v,u), A(u)\right) \le A(v) \\
&\Leftrightarrow&  \varphi^{-1}(\max(\varphi(\widetilde{R}(u,v)) + \varphi(\hat{A}(v)) - 1, 0) \leq  \hat{A}(u) \\
&\Leftrightarrow&  \max(\varphi(\widetilde{R}(u,v)) + \varphi(\hat{A}(v)) - 1, 0) \leq  \varphi(\hat{A}(u)) \\
&\Leftrightarrow&  \max(\widetilde{R}_\varphi(u,v) + \alpha_v - 1, 0) \leq  \alpha_u \\
&\Leftrightarrow& \widetilde{R}_{\varphi}(u,v) \leq \alpha_u - \alpha_v + 1
\end{eqnarray*}
where we introduced the shorthands $\widetilde{R}_{\varphi}(u,v) = \varphi(\widetilde{R}(u,v))$, $\alpha_u = \varphi(\hat{A}(u))$ and $\alpha_v = \varphi(\hat{A}(v))$. The last equivalence holds because 0 is always smaller than $\alpha_u$, hence the $\max$ operator can be lifted. 

If $T_{P, \varphi}$ is used then in an analogous way we find
$$
 \varphi^{-1}(\varphi(\widetilde{R}(u,v)) \varphi(\hat{A}(v))) \leq  \hat{A}(u) \Leftrightarrow \alpha_v \widetilde{R}_{\varphi}(u,v)  \leq \alpha_u
$$
for all $u,v \in U$.

The border constraints now become $0 \leq \alpha_u \leq 1$ for all $u \in U$. We can conclude that using continuous Archimedean $t$-norms leads to linear optimization constraints. This is a promising result since many optimization solvers are very efficient with linear constraints. 

In both cases, the empirical risk can be expressed as
$$
\sum_{u \in U} L(\Bar{A}(u), \varphi ^{-1}(\alpha_u)).
$$
In the empirical risk above, the non-linear term $\varphi ^{-1}(\alpha_u)$ appears. Function $\varphi ^{-1}$ is an arbitrary bijection which can lead to a non-convex optimization problem. However, Proposition \ref{prop:quantile_scaling} states that a different scaling of values does not change the Bayes predictor delivered by the $p$-quantile loss function. To eliminate the non-linearity, we can apply $\varphi$ to both parameters of the loss function 
and replace $L_p(\Bar{A}(u), \varphi ^{-1}(\alpha_u))$ by $L_p(\varphi(\Bar{A}(u)), \alpha_u)$. Although the value of the estimand (the quantity that is estimated, i.e., the Bayes predictor $A_p$) remains unchanged with the new loss function, the estimator (the result of the optimization $\hat{A}_p$) can be different. From the theory of quantile regression, we can express the optimization of the quantile risk as a linear program \cite{koenker2001quantile}. We introduce new variables $x_u, u \in U$ and $y_u, u \in U$ such that $x_u = \max(\varphi(\Bar{A}(u) - \alpha_u ), 0)$, $y_u =  \max( \alpha_u - \varphi(\Bar{A}(u)), 0)$, as well as the shorthand $\Bar{A}_{\varphi}(u) = \varphi(\Bar{A}(u))$. In case $T_{L, \varphi}$ is used, we can reformulate optimization problem (\ref{eq:general_optimization}) as
\begin{equation}
\label{eq:T_L_linear_program}
\begin{aligned}
&\text{minimize}   && p\displaystyle\sum_{u \in U}x_u + (1-p)\displaystyle\sum_{u \in U}y_u, &&\\
&\text{subject to}    &&\alpha_u - \alpha_v + 1 \geq \widetilde{R}_{\varphi}(u,v),  && u, v \in U\\
        &    &&x_u - y_u =\Bar{A}_{\varphi}(u) - \alpha_u,  && u \in U\\
  &              &&0 \leq \alpha_u \leq 1, \, x_u \geq 0, \, y_u\geq 0. &&u \in U
\end{aligned}
\end{equation}
In case of $T_{P, \varphi}$, optimization problem (\ref{eq:general_optimization}) obtains the form
\begin{equation}
\label{eq:T_P_linear_program}
\begin{aligned}
&\text{minimize}    && p\displaystyle\sum_{u \in U}x_u + (1-p)\displaystyle\sum_{u \in U}y_u, &&\\
&\text{subject to}    &&\alpha_v \widetilde{R}_{\varphi}(u,v)  \leq \alpha_u,  && u, v \in U\\
        &    &&x_u - y_u = \Bar{A}_{\varphi}(u) - \alpha_u,  && u \in U\\
  &              &&0 \leq \alpha_u \leq 1,  \, x_u \geq 0, \, y_u\geq 0. &&u \in U
\end{aligned}
\end{equation}
Summarizing, for quantile risk and a continuous Archimedean $t$-norm, the optimization problem (\ref{eq:general_optimization}) can be expressed as a linear program and, therefore, efficiently solved using one of many existing efficient linear programming solvers. We have the following technical result.
\begin{proposition}
\label{prop:redundant_constraints}
Constraints $0 \leq \alpha_u \leq 1, u \in U$ in (\ref{eq:T_L_linear_program}) and (\ref{eq:T_P_linear_program}), are redundant.
\end{proposition}
\begin{proof}
Assume that the constraints are removed and that an optimal solution $\alpha^*_u, u \in U$, has values smaller than 0 or larger than 1. We construct another solution from $\alpha^*_u, u \in U$, by replacing values larger than 1 by 1, and values smaller than 0 by 0. It is easy to check that the new solution satisfies the granular representability constraints. From the constraints $x_u - y_u = \Bar{A}_{\varphi}(u) - \alpha_u,  u \in U$, it is easy to see that when $\alpha_u \geq 1$ then $\Bar{A}_{\varphi}(u) - \alpha_u \leq 0$, which leads to $x_u = 0$ and $y_u = \alpha_u - \Bar{A}_{\varphi}(u)$, and when $\alpha_u \leq 0$ then $\Bar{A}_{\varphi}(u) - \alpha_u \geq 0$, which leads to $x_u = \Bar{A}_{\varphi}(u) - \alpha_u$ and $y_u = 0$. Hence, after replacing values larger than 1 by 1, the values of $y_u$ will be reduced and after replacing values smaller than 0 by 0, the values of $x_u$ will also be reduced. In both cases, the value of the objective function will be reduced. Therefore, we constructed a feasible solution with a smaller cost which contradicts the optimality of $\alpha^*_u, u \in U$.
\end{proof}

A solution of optimization problems (\ref{eq:T_L_linear_program}) and (\ref{eq:T_P_linear_program}) is not necessarily unique. The reason is that the estimated CDF of $\mathcal{A}(u)$ for some $u\in U$ is not necessarily increasing and may have a region where it is constant. In such case, we obtain an uncountable family of optimal solutions. However, if for some probability parameter $p$ we have infinitely many solutions, the lower and upper bounds of such family of solutions can be calculated by running the linear programs with parameters $p-\epsilon$ and $p+\epsilon$, respectively, for sufficiently small $\epsilon$.

If the mean squared error is used as a loss function, it is obvious that the objective function will become non-linear. Also,
Proposition \ref{prop:quantile_scaling} does not hold anymore and using $L_p(\Bar{A}_{\varphi}(u), \alpha_u)$ instead of $L_p(\Bar{A}(u), \varphi ^{-1}(\alpha_u))$ will lead to the estimation of a different Bayes predictor. However, we will include this approach in our analysis since it may give good results in practical applications. In this case, the optimization problem for the $t$-norm $T_{L,\varphi}$ is
\begin{equation}
\label{eq:T_L_quad_program}
\begin{aligned}
&\text{minimize}  && \displaystyle\sum_{u \in U} (\alpha_u - \Bar{A}_{\varphi}(u))^2, &&\\
&\text{subject to}    &&\alpha_u - \alpha_v + 1 \geq \widetilde{R}_{\varphi}(u,v),  && u, v \in U\\
  &              && x_u \geq 0, \, y_u\geq 0, &&u \in U
\end{aligned}
\end{equation}
while for $T_{P,\varphi}$ the corresponding problem is
\begin{equation}
\label{eq:T_P_quad_program}
\begin{aligned}
&\text{minimize}  && \displaystyle\sum_{u \in U}  (\alpha_u - \Bar{A}_{\varphi}(u))^2, &&\\
&\text{subject to}    &&\alpha_v \widetilde{R}_{\varphi}(u,v)  \leq \alpha_u,  && u, v \in U\\
  &              &&  x_u \geq 0, \, y_u\geq 0. &&u \in U
\end{aligned}
\end{equation}
Using a similar argument as in Proposition \ref{prop:redundant_constraints}, we may drop the constraints $0 \leq \alpha_u \leq 1, u \in U$.

The optimization problems that involve the mean squared error loss function can be solved using the simplex method for quadratic programming. The solution in this case is unique. The uniqueness can be observed geometrically, i.e., the optimal solution is the point of contact of the convex simplex formed by the linear constraints and the nonlinear strictly convex cone formed by the objective function. Convex and strictly convex shapes can touch each other in only one point. 
    
In the appendix of this paper, we investigate if the linear optimization problems can be solved using non-analytical (combinatorial) approaches. Namely, the dual versions of problems (\ref{eq:T_L_linear_program}) and (\ref{eq:T_P_linear_program}) can be modeled as the min-cost flow problem and its variations. We recall the min-cost flow problem and some algorithms used to solve it in \ref{sec_app:min_cost_flow} while we show how to model dual problems of (\ref{eq:T_L_linear_program}) and (\ref{eq:T_P_linear_program}) as the min-cost flow problem and a variation of the min-cost flow problem, respectively, in \ref{sec_app:duality}. In the same section, we provide a greedy algorithm to solve the aforementioned variation based on the algorithm that solves the original min-cost flow problem. Since the algorithm is new, we provide its proof of correctness in \ref{sec_app:proof_of_correctness}.


\section{Properties}
\label{sec:properties}

In this section, we prove some properties of the granular approximations obtained in Section \ref{sec:calculation}. First we show that the proposed approach is a generalization of the KS approach when the number of ordinal labels is 2.
\begin{proposition}
If $\widetilde{R}$ and $\Bar{A}$ are crisp, then Problem (\ref{eq:general_optimization}) is reduced to Problem (\ref{eq:monotone_approximation}) for $K=2$.
\end{proposition}

\begin{proof}
If $\Bar{A}$ is crisp, it is obvious that the objective function from (\ref{eq:general_optimization}) corresponds to the objective function from (\ref{eq:monotone_approximation}) for $K=2$, where the labels with value $1$ are those that are more preferred. Regarding the constraints, we examine the granular conditions in form $\widetilde{R}(u,v) \leq I(\hat{A}(v), \hat{A}(u))$. If $\widetilde{R}(u,v) = 0$, then there are no restrictions on the implication, i.e., we do not have a constraint. If $\widetilde{R}(u,v) = 1$ then $\hat{A}(v) \leq \hat{A}(u)$ from the ordering property of $I$ (\ref{eq:ordering_property}). Since $\widetilde{R}(u,v) = 1$ means that $u \succeq v$ ($u$ dominates $v$) then the condition $\widetilde{R}(u,v) = 1 \Rightarrow \hat{A}(u) \geq \hat{A}(v)$ is equivalent to $ u \succeq v \Rightarrow \hat{A}(u) \geq \hat{A}(v)$ which is exactly the condition from (\ref{eq:monotone_approximation}).
\end{proof}

In the remainder of this section, we prove some properties of our approach that also hold in the original (crisp) one.

For granular approximations obtained with the $p$-quantile loss, the monotonicity property holds.
\begin{proposition}
\label{prop:parameter_monotonicity}
Let $p$ and $q$ be two real numbers from the unit interval and let $\hat{A}_p$ and $\hat{A}_q$ be the outputs of the optimization problem (\ref{eq:T_L_linear_program}) or (\ref{eq:T_P_linear_program})  with $p$ and $q$ as probability parameters. It holds that
$$
p \leq q \Rightarrow \forall u \in U, \hat{A}_p(u) \leq \hat{A}_q(u).
$$
\end{proposition}
\begin{proof}
The proof is provided in \ref{sec_app:proof_parameter_monotonicity}. It relies on the greedy combinatorial approach presented in the previous sections of Appendix, hence those previous sections are necessary for the understanding of the proof. 
\end{proof}
Before we introduce the next property, we present the following lemma.
\begin{lemma}
If fuzzy set $A$ is granularly representable w.r.t. $T$-preorder relation $\widetilde{R}$, then $coA$ is granularly representable w.r.t.  $\widetilde{R}^{-1}$.
\end{lemma}
\begin{proof}
For $A$ being granularly representable, we have
$$
T(\widetilde{R}(u,v), A(v)) \leq A(u).
$$
Applying negation $N$ on both sides of the inequality, we have
\begin{align*}
T(\widetilde{R}(u,v), A(v)) \leq A(u) &\Rightarrow  N(T(\widetilde{R}(u,v), A(v))) \geq  N(A(u)) \\ &\Leftrightarrow I(\widetilde{R}(u,v), coA(v))  \geq coA(u) \\ &\Leftrightarrow
T(coA(u), \widetilde{R}(u,v)) \leq coA(v) \\ &\Leftrightarrow
T(\widetilde{R}^{-1}(v,u), coA(u)) \leq coA(v).
\end{align*}
The first equivalence follows from Proposition (\ref{prop:residual_property}) while the second is the residuation property. 
\end{proof}
In the previous proposition, implication becomes an equivalence if we use IMTL triplets as operators. 

For the fuzzy rough approximations that are obtained with IMTL operators, we have the well known duality property as stated in Subsection \ref{subsec:fuzzy_logic}. That property can be extended to the granular approximations. 

\begin{proposition}
\label{prop:granular_duality}
Let $\alpha_u^*, u \in U$ be a minimizer of the optimization problem (\ref{eq:general_optimization}) with nilpotent $t$-norm $T_{L, \varphi}$, relation $\widetilde{R}$, observations $\Bar{A}$ and risk $\sum_{u \in U} L_p(\varphi(\Bar{A}(u)), \alpha_u)$ (for short $L_p$ problem). Then $1-\alpha^*_u, u \in U$, is a minimizer of the optimization problem (\ref{eq:general_optimization}) with the same $t$-norm, relation $\widetilde{R}^{-1}$, observations $\Bar{A}$ and risk \\
$\sum_{u \in U} L_{1-p}(\varphi(co\Bar{A}(u)), \alpha_u)$ (for short $L_{1-p}$ problem).
\end{proposition}
\begin{proof}
Let $\alpha_u, u \in U$, be a feasible solution of the $L_p$ problem, i.e., it satisfies granularity conditions w.r.t. relation $\widetilde{R}$
$$
\alpha_u - \alpha_v + 1 \geq \varphi(\widetilde{R}(u,v)).
$$
The expression above is equivalent to 
$$
(1 - \alpha_v) - (1 - \alpha_u) + 1 \geq \varphi(\widetilde{R}^{-1}(v,u)),
$$
which states that $1 - \alpha_u, u \in U$ satisfies granularity conditions w.r.t. relation $\widetilde{R}^{-1}$ and, therefore, it is a feasible solution of the $L_{1-p}$ problem. We observe that $\varphi(co\Bar{A}(u)) = \varphi(\varphi^{-1}(1 - \varphi(\Bar{A}^*(u)))) = 1 - \varphi(\Bar{A}^*(u))$.
Regarding the loss function, we have that
\begin{align*}
&\sum_{u \in U} L_p(\varphi(\Bar{A}(u)), \alpha_u) = \\ &= \begin{cases}
        p|\varphi(\Bar{A}(u)) - \alpha_u| & \text{if   }  \varphi(\Bar{A}(u)) - \alpha_u \geq 0, \\
        (1-p)|\varphi(\Bar{A}(u)) - \alpha_u| & \text{if   }  \alpha_u - \varphi(\Bar{A}(u))  \geq 0,
        \end{cases} \\
        &= \begin{cases}
        p|(1-\alpha_u)-(1-\varphi(\Bar{A}(u))) | & \text{if   }  (1-\alpha_u)-(1-\varphi(\Bar{A}(u))) \geq 0, \\
        (1-p)|(1-\alpha_u) -(1-\varphi(\Bar{A}(u)))| & \text{if   }  (1-\varphi(\Bar{A}(u))) - (1-\alpha_u) \geq 0,
        \end{cases} \\
        &= \begin{cases}
        (1-p)|\varphi(co\Bar{A}(u))) - (1-\alpha_u) | & \text{if   }  \varphi(co\Bar{A}(u)) - (1-\alpha_u) \geq 0,\\
        p|\varphi(co\Bar{A}(u)) - (1-\alpha_u) | & \text{if   }  (1-\alpha_u)-\varphi(co\Bar{A}(u)) \geq 0,
        \end{cases} \\
        &= L_{1-p}(\varphi(co\Bar{A}(u)), 1-\alpha_u).
\end{align*}
Since $\alpha_u, u \in U$, is a feasible solution of the $L_p$ problem if and only if $1 - \alpha_u, u \in U$, is a feasible solution of the $L_{1-p}$ problem, and due the previous equality, we finish the proof.
\end{proof}

Since the optimal fuzzy set $\hat{A}^*$ of the $L_p$ problem is calculated as $\hat{A}^*(u) = \varphi^{-1}(\alpha^*_u)$, then the optimal fuzzy set of the $L_{1-p}$ is $\varphi^{-1}(1 - \alpha^*_u) = \varphi^{-1}(1 - \varphi(\hat{A}^*(u))) = co\hat{A}^*(u)$, i.e., we have the duality.

The duality also holds for the mean squared error risk. The proof is very similar to the proof of Proposition
\ref{prop:granular_duality} where the only difference is that the loss function stays the same in the dual problems. 

The next step is to examine fuzzy rough lower and upper approximations in the context of the newly defined optimization problems. In the crisp case, the lower and upper fuzzy rough approximations are seen as sets of necessary and possible knowledge respectively. In other words, the actual ill-known knowledge must contain the lower approximation and be contained in the upper one. In probabilistic terms, the probability that the actual knowledge is between these approximations is 1 \cite{palangetic2020rough}. Hence, the lower and upper approximations are the extreme values in the probability distributions of the actual knowledge. It means that the lower approximation is the 0-quantile while the upper approximation is the 1-quantile. We have the following proposition.
\begin{proposition}
The respective lower fuzzy rough approximations are solutions of the optimization problems (\ref{eq:T_L_linear_program}) and (\ref{eq:T_P_linear_program}) for probability parameter $p=0$ while the respective upper fuzzy rough approximations are solutions of the same problems for probability parameter $p=1$.
\end{proposition}

\begin{proof}
When optimization problems (\ref{eq:T_L_linear_program}) and (\ref{eq:T_P_linear_program}) are considered in terms of $\alpha$ and not in terms of $\hat{A}$, they can be seen as problem (\ref{eq:general_optimization}) with $t$-norm $T_L$ or $T_P$, relation $\widetilde{R}_{\varphi}$ and observations $\Bar{A}_{\varphi}$. If $p=1$, then the loss function for $u \in U$ is equal 0 if $\alpha_u - \Bar{A}_{\varphi}(u) \geq 0$ and to a positive value otherwise. If for all $u \in U $ it holds that $\alpha_u \geq \Bar{A}_{\varphi}(u)$, then the objective is 0 and hence any such $\alpha$ is a solution. Such fuzzy set $\alpha$ contains fuzzy set $\Bar{A}_{\varphi}$ and is granularly representable w.r.t. $t$-norm $T_L$ or $T_P$ and relation $\widetilde{R}_{\varphi}$. From Proposition \ref{prop:l_and_s_gp}, the smallest such $\alpha$ is the fuzzy rough upper approximation, i.e., the smallest solution is
$$
\alpha_u^{*} = \max_{v \in U} T_L(\widetilde{R}_{\varphi}(v, u), \Bar{A}_{\varphi}(v)),
$$
or with $T_P$ instead of $T_L$. Then, the final solution $\hat{A}^*$ is obtained
\begin{align*}
\hat{A}^*(u) &= \varphi^{-1}(\alpha_u ^{*}) \\
&= \varphi^{-1}(\max_{v \in U} T_L(\widetilde{R}_{\varphi}(v, u), \Bar{A}_{\varphi}(v))) \\
&= \max_{v \in U}\varphi^{-1}(T_L(\varphi(\widetilde{R}(v, u)), \varphi(\Bar{A}(v)))) \\
&= \max_{v \in U}  T_{L,\varphi}(\widetilde{R}(v, u), \Bar{A}(v)) = \overline{\text{apr}}_{\widetilde{R}}^{\max, T_{L, \varphi}}(A)(u).
\end{align*}
The derivation for $T_P$ is the same.

The proof for the lower approximation is analogous. 
\end{proof}

\section{Didactic Examples}
\label{sec:toy_examples}

In this section, we provide two didactic examples of the approaches (\ref{eq:T_L_linear_program}) and (\ref{eq:T_L_quad_program}) that use $T_{L, \varphi}$ $t$-norms. One example is related to a binary classification problem while the other one is related to a regression problem. 

To evaluate pairwise fuzzy relation values, in both cases we use a $T_L$-equivalence relation called triangular similarity. For a condition attribute $q$, it is defined as
\begin{align}
\label{eq:triangular_similarity}
   R_q (u,v) = \max \left ( 1 - \frac{|f(u,q) - f(v,q)|}{range(q)},0\right) 
\end{align}
while the overall relation is then $R(u,v) = \min_{q} R_q(u,v)$. More details on such similarity relation are provided in \cite{palangetic2021fuzzy}. 

For the classification purpose, we use 4 instances from two classes of the well-known iris dataset which can be found in the UCI dataset repository \cite{ucidatasets2019}. Those instances are shown in Table \ref{tab:toy_data_classification}.

\begin{table}[H]
\centering
\begin{tabular}{l|llll|l}
\toprule
instance &   att1  &    att2 &  att3   & att4 &  decision \\ \hline
1 & 5.4 & 3.4 & 1.7 & 0.2 & 0 \\
2 & 4.4 & 3.2 & 1.3 & 0.2 & 0 \\
3 & 5.9 & 3   & 4.2 & 1.5 & 1 \\
4 & 6.3 & 2.3 & 4.4 & 1.3 & 1 \\
\bottomrule
\end{tabular}
\caption{Classification data example}
\label{tab:toy_data_classification}
\end{table}

In this binary classification example, the observed values are from the set $\{0,1\}$ where they indicate if an instance is in a decision class or not. We consider such crisp set as a fuzzy one in order to obtain granular approximations. 
After applying the relation on every pair of instances, we get a relation matrix shown in Table \ref{tab:rel_values_classification}.

\begin{table}[H]
\centering
\begin{tabular}{l|llll}
 \begin{tabular}[c]{@{}l@{}}instance vs.\\ instance\end{tabular} & 1     & 2     & 3     & 4     \\ \hline
1 & 1     & 0.917 & 0.525 & 0.208 \\
2 & 0.917 & 1     & 0.492 & 0.292 \\
3 & 0.525 & 0.492 & 1     & 0.667 \\
4 & 0.208 & 0.292 & 0.667 & 1    
\end{tabular}
\caption{Matrix of relation values for the classification case}
\label{tab:rel_values_classification}
\end{table}
In Table \ref{tab:rel_values_classification}, the first row and the first column stand for indices of instances from Table \ref{tab:toy_data_classification}. The remaining entries are values of the relation on the corresponding pair of instances.    

After the matrix is calculated, we pass it together with the decision attribute to the optimization problem (\ref{eq:T_L_linear_program}) with probability parameters $p \in \{0, 0.25, 0.5, 0.75, 1\}$. The obtained granular approximations are given in Table \ref{tab:granular_approx_quantile}.

\begin{table}[H]
\centering
\begin{tabular}{l|llll}
\begin{tabular}[c]{@{}l@{}}$p$ vs.\\ instance\end{tabular}& 1     & 2     & 3     & 4     \\ \hline
0    & 0     & 0     & 0.475 & 0.708 \\
0.25 & 0     & 0     & 0.475 & 0.708 \\
0.5  & 0.326 & 0.292 & 0.8   & 1     \\
0.75 & 0.525 & 0.492 & 1     & 1     \\
1    & 0.525 & 0.492 & 1     & 1    
\end{tabular}
\caption{Granular approximations in the classifications case for the $p$-quantile loss}
\label{tab:granular_approx_quantile}
\end{table}

In every row, we have a granular approximation for a corresponding probability parameter from the first column. Every entry is a fuzzy membership degree for the corresponding instance which may be interpreted as the degree up to which the instance belongs to class with label 1. Since that fuzzy value is unknown, we have its distribution characterized with quantiles. For example, in the second row of Table \ref{tab:granular_approx_quantile}, we say that with probability 0.25, the degree up to which instance 3 belongs to the class with label 1 is not greater than 0.475.

The results for the mean squared error used in (\ref{eq:T_L_quad_program}) are shown in Table \ref{tab:granular_approx_squared}.
\begin{table}[H]
\centering
\begin{tabular}{l|llll}
instance & 1     & 2     & 3     & 4     \\ \hline
degree & 0.221 & 0.187 & 0.696 & 0.896
\end{tabular}
\caption{Granular approximations in the classifications case for the mean squared error}
\label{tab:granular_approx_squared}
\end{table}
In this case, we may say that the expected degree to which instance 3 belongs to the class with label 1 is 0.696.

In this particular case, the granularity (consistency) condition expressed in form $T(\widetilde{R}(u, v), \hat{A}(v)) \leq \hat{A}(u) $ can be interpreted as ``if instance $v$ belongs to the decision class, and $u$ is similar to $v$, then $u$ also belongs to the same decision class". We express it numerically for instances 1 and 3, for which the similarity can be obtained from Table \ref{tab:rel_values_classification} as 0.525. We notice that the granularity condition was not satisfied with their original labels from Table \ref{tab:toy_data_classification} since it holds that $T(0.525, 1) = 0.525 > 0$. Using the new labels from Table \ref{tab:granular_approx_squared}, we have that the condition is satisfied since $T(0.525, 0.696) = 0.221 \leq 0.221$ and $T(0.525, 0.221) = 0 \leq 0.696$.


In the regression example, we use 'Real estate valuation' dataset that can also be found in the UCI dataset repository \cite{ucidatasets2019}. The goal of this dataset is to predict a price of a real estate given its features like position, age, closeness to some important facilities, etc. We use 5 instances displayed in Table \ref{tab:toy_data_regression}. Since the decision values are real estate prices, we have to model them as a fuzzy set in order to use our methods. Fuzzy predicate ``expensive" can be modeled using prices, hence we construct a new fuzzy set  which indicates how expensive is a certain real estate based on the prices. The new fuzzy set is constructed using the following formula:
$$
fuzzy\_decision(x) = \max(\min(\frac{decision(x)- Q(decision, .005)}{Q(decision, .995) - Q(decision, .005)}, 1), 0),
$$
where $Q$ stands for $p$-quantile. This is similar to the linear scaling to the $[0,1]$ range, just here using quantiles we achieve that 0.5 percent of smallest prices has expensiveness 0 and 0.5 percent of largest prices has expensiveness 1. We use this approach in order to handle extreme values - outliers. After applying this transformation on the complete dataset, the fuzzy values of our chosen 5 instances are given in the last column of Table \ref{tab:toy_data_regression}.
\begin{table}[H]
\begin{tabular}{l|llllll|l|l}
\toprule
instance  & att1    & att2 & att3     & att4 & att5    & att6     & decision & \begin{tabular}[c]{@{}l@{}}fuzzy\\ decision\end{tabular} \\ \hline
1 & 2013.42 & 8.4  & 1962.628 & 1    & 24.955 & 121.555 & 23.5     & 0.18                                                     \\
2 & 2013.58 & 13.3 & 561.9845 & 5    & 24.987 & 121.544 & 47.3     & 0.54                                                     \\
3 & 2013.42 & 17.9 & 1783.18  & 3    & 24.967 & 121.515 & 22.1     & 0.158                                                    \\
4 & 2013.42 & 0    & 292.9978 & 6    & 24.977 & 121.545 & 73.6     & 0.938                                                    \\
5 & 2013.08 & 17.5 & 395.6747 & 5    & 24.957 & 121.534  & 24.5     & 0.195 
\\
\bottomrule                                                
\end{tabular}
\caption{Regression data example}
\label{tab:toy_data_regression}
\end{table}

Applying the relation on every pair of instances leads to the relation matrix shown in Table \ref{tab:rel_values_regression}. 

\begin{table}[H]
\centering
\begin{tabular}{l|lllll}
 \begin{tabular}[c]{@{}l@{}}instance vs.\\ instance\end{tabular} & 1     & 2     & 3     & 4     & 5     \\ \hline
1 & 1     & 0.6   & 0.569 & 0.5   & 0.6   \\
2 & 0.6   & 1     & 0.687 & 0.696 & 0.454 \\
3 & 0.569 & 0.687 & 1     & 0.591 & 0.635 \\
4 & 0.5   & 0.696 & 0.591 & 1     & 0.6   \\
5 & 0.6   & 0.454 & 0.635 & 0.6   & 1   
\end{tabular}
\caption{Matrix of relation values for the regression case}
\label{tab:rel_values_regression}
\end{table}

As before, we pass matrix values with the new expensiveness values to optimization problem (\ref{eq:T_L_linear_program}) with probability parameters $p \in \{0, 0.25, 0.5, 0.75, 1\}$. The obtained granular approximations are given in Table \ref{tab:granular_approx_quantile_regression}.

\begin{table}[H]
\centering
\begin{tabular}{l|lllll}
  \begin{tabular}[c]{@{}l@{}}$p$ vs.\\ instance\end{tabular}   & 1     & 2     & 3     & 4     & 5     \\ \hline
0    & 0.18  & 0.472 & 0.158 & 0.567 & 0.195 \\
0.25 & 0.18  & 0.472 & 0.158 & 0.567 & 0.195 \\
0.5  & 0.18  & 0.54  & 0.226 & 0.594 & 0.195 \\
0.75 & 0.343 & 0.54  & 0.435 & 0.843 & 0.444 \\
1    & 0.438 & 0.634 & 0.529 & 0.938 & 0.538
\end{tabular}
\caption{Granular approximations in the regression case for the $p$-quantile loss}
\label{tab:granular_approx_quantile_regression}
\end{table}

The obtained fuzzy values are estimations of quantiles of the expensiveness, under the assumption that it is a random fuzzy set and that its realizations are given in Table \ref{tab:toy_data_regression}. We interpret the values in a way that, for example, in the third row of Table \ref{tab:granular_approx_quantile_regression} we say that the expensiveness of instance 2 is less than 0.54 with probability 0.5, or in the fourth row of the table, we say that the expensiveness of instance 4 is less than 0.843 with probability 0.75. 

The results for the mean squared error used in (\ref{eq:T_L_quad_program}) are shown in Table \ref{tab:granular_approx_squared_regression}.
\begin{table}[H]
\centering
\begin{tabular}{l|lllll}
instance & 1     & 2    & 3     & 4     & 5     \\ \hline
degree & 0.195 & 0.54 & 0.286 & 0.695 & 0.295
\end{tabular}
\caption{Granular approximations in the classifications case for the mean squared error}
\label{tab:granular_approx_squared_regression}
\end{table}
In this case, we may say that the expected expensiveness of instance 4 is equal to 0.695.

In the regression case, the granularity (consistency) condition expressed in form $T(\widetilde{R}(u, v), \hat{A}(v)) \leq \hat{A}(u) $ is interpreted as ``if instance $v$ is expensive and $u$ is similar to $v$, then $u$ is also expensive". We observe instances 2 and 4 for which the similarity is equal to 0.696. Their original labels from Table \ref{tab:toy_data_regression} do not satisfy the granularity condition since $T(0.696, 0.938) = 0.634 > 0.54$. Using the new labels from Table \ref{tab:granular_approx_squared_regression}, we have that $T(0.696, 0.695) = 0.391 \leq 0.54$ and $T(0.696, 0.54) = 0.236 \leq 0.695$. Hence, the condition is satisfied.


\section{Conclusion and future work}
In this paper, we introduced a novel machine learning approach for handling inconsistencies in prediction problems with respect to a fuzzy relation. Our work was motivated by the method introduced by Kotłowski and Słowiński \cite{kotlowski2008statistical} for handling monotone inconsistency and we showed that the novel approach is a generalization of the same method in the binary classification case. Using fuzzy relations, the novel method is able to handle gradual relationships among instances while the KS approach can distinguish only two cases: either instances relate or not.

The novel approach produces a granular approximation of a fuzzy set. The approximation is granularly representable (without inconsistencies) and as close as possible to the original fuzzy set (w.r.t. a given loss function). It can be seen as a fuzzy counterpart of the monotone approximation produced by the KS approach. As in the work of Kotłowski and Słowiński, we provided statistical foundations of the granular approximations. In the next step, we formulated optimization problems in order to calculate the approximations and we showed some of their important properties.
At the end, we provided two didactic examples; one for a binary classification problem and one for a regression problem. In the didactic examples, we showed how fuzzy relations are used to model relationship among numerical data, how the granular approximations are calculated and how to interpret them in the two cases for different loss functions. 

Since this work is theoretical in its nature, for the future work we propose its implementation and verification on practical applications. As it was already mentioned in the Introduction, the possible applications are in fuzzy rough set based methods and fuzzy rule induction. 
Therefore, our future work will mainly concentrate on the aforementioned implementation and computational experiments. 

\label{sec:conclusion}

\section*{Acknowledgements}

Marko Palangetić and Chris Cornelis would like to thank Odysseus project from Flanders Research Foundation (FWO) for funding their research.
Salvatore Greco wishes to acknowledge the support of the Ministero dell’Istruzione, dell’Universitá e della Ricerca (MIUR) - PRIN 1576 2017, project “Multiple Criteria Decision Analysis and Multiple Criteria Decision Theory”, grant 2017CY2NCA. 
Roman Słowiński is acknowledging the support of grant 0311/SBAD/0700.



\bibliographystyle{elsarticle-num} 
\bibliography{bibliography.bib}

\appendix

\section{Minimum-cost flow problem}
\label{sec_app:min_cost_flow}
This section is based on the monograph \cite{ahuja1988network}, especially on its 9th chapter.

A flow network is defined as a directed graph where a real value called imbalance is assigned to each node. Imbalances split nodes into two subsets: supply nodes with a positive imbalance (supply value) and demand nodes with a negative imbalance (demand value). Moreover, each edge is characterized by a positive real capacity, and a cost value. We also assign flow amounts to each edge which satisfy the condition that they are at most as large as capacities. More formally, let $G$ be a finite set of nodes, $E \subseteq G \times G$ the finite set of edges, while $F = (G,E)$ is the flow network. We denote imbalances with $b_i$ for $i \in G$, capacities with $l_{i,j}$, costs with $c_{i, j}$ and flow with $z_{i,j}$ for $(i, j) \in E$.

The minimum-cost flow problem is an optimization problem defined on a flow network where we want to transport flow from the supply nodes to the demand
nodes, such that
\begin{itemize}
    \item[--] the difference between the flow that leaves a node and the flow that enters the node is equal to the imbalance of this node,
    \item[--] a flow in a particular edge is at most as large as the capacity of that edge, and
    \item[--] the total cost of the flow transportation is minimal.
\end{itemize}
Formally, we have the following problem:

\begin{subequations}
\label{eq:min_cost_flow}
\begin{align}
&\text{minimize}   &&\displaystyle\sum_{(i,j) \in  E} c_{i,j} z_{i, j}, &&\label{eq:mcf_objective} \\
&\text{subject to}    &&\sum_{j:(i,j) \in E} z_{i,j} - \sum_{j:(j,i) \in E} z_{j,i} = b_i,  && i \in G \label{eq:mcf_balance_constr}\\    
    &    &&0 \leq z_{i,j} \leq l_{i, j}. &&  (i,j) \in E& \label{eq:mcf_capacity_constr}
\end{align}
\end{subequations}

We distinguish two sets of constraints in the previous optimization problem: balance constraints (\ref{eq:mcf_balance_constr}) and capacity constraints (\ref{eq:mcf_capacity_constr}). If we sum the balance constraints, we get $\sum_{i \in G} b_i = 0$ which states that the amount of supply is equal to the amount of demand, which is a necessary assumption to have a feasible solution. 

We say that a flow is feasible if it is a feasible solution of (\ref{eq:min_cost_flow}), while we say that we have a pseudo-flow if only the capacity constraints are satisfied.

For a given pseudo-flow $z'$, a residual network $F' = (G, E')$ can be defined. We have new imbalances:
$$
b'_{i} = b_i -\bigg(\sum_{j:(i,j) \in E} z_{i,j} - \sum_{j:(j,i) \in E} z_{j,i}\bigg ),
$$
while for each edge $(i,j) \in E$ for which $z'_{i,j} > 0$, we add the reverse edge $(j,i)$ to the network with cost $c'_{j,i} = -c_{i,j}$, while keeping the original edge. The capacity of the original edge $(i,j)$ in $F'$ is $l'_{i,j} = l_{i,j} - z_{i,j}$, while the capacity of the added reverse edge $(j,i)$ is $l'_{j,i} = z_{i,j}$. We may notice that when adding a new edge $(j,i)$ to $E'$, there can already exist an edge $(j,i)$ from $E$. However, in our case of use, we will not face such an issue, i.e., we will have either $(i,j)$ or $(j,i)$ in $E$ and not both at the same time. The residual network keeps the complete information about flow $z'$ which can be reconstructed from $F'$.

The concept of residual network is important for the development of algorithms for solving (\ref{eq:min_cost_flow}). In this moment, we will not discuss the existence of a feasible solution in general since later we will show that it always exists in our case of use.

A cost of a particular path or cycle in the flow network is calculated as the sum of the costs of edges in that path or cycle. For an optimal flow $z^*$, we have the following result.
\begin{proposition}
\label{prop:min_cost_neg_cycle}
A flow $z^*$ is optimal if and only if there are no cycles of negative cost in the residual network $F(z^*)$.
\end{proposition}
Bearing in mind Proposition \ref{prop:min_cost_neg_cycle}, a simple algorithm can be constructed to solve (\ref{eq:min_cost_flow}). Namely, we construct an initial feasible flow in our network, then search for the negative cycles and eliminate them.

However, a more useful algorithm for us is the Successive Shortest Path (SSP) algorithm for solving the minimum-cost flow problem. The algorithm is provided as Algorithm \ref{algo:ssp}.

\begin{algorithm}[H]
    \caption{Successive Shortest Paths}
\begin{algorithmic}[1]
\State \textbf{Input:} Flow network $F = (G,E)$.
\State \textbf{Output:} Flow $z$.
\State{Set initial flow $z_{i,j} = 0$, $(i,j) \in E$}
\State{Set initial residual network $F' = F$}
\While {there exist supply/demand values different from 0}
    \State{Pick supply node $i$ and demand node $j$}
    \State{Calculate the shortest path $P$ from $i$ to $j$ using cost values from $F'$}
    \State{Send the largest possible amount of flow through $P$}
    \State{Update $F'$}
\EndWhile
\State{Reconstruct $z$ from $F'$}
\end{algorithmic}
\label{algo:ssp}
\end{algorithm}

The shortest path $P$ can be calculated using the Bellman-Ford algorithm since $F'$ may contain negative values. The largest possible amount of flow through $P$ is calculated as $\delta = \min\{b'_i, |b'_j|, c'_{i_1,j_1} \text{ for } (i_1,j_1) \in P \}$. The residual network is then updated such that 
\begin{itemize}
    \item $b'_i = b'_i - \delta$, $ b'_j = b'_j + \delta $
    \item $c'_{i,j} = c'_{i,j} - \delta$,  $c'_{j,i} = c'_{j,i} + \delta$ for $(i,j) \in P$
\end{itemize}
The idea of the proof of correctness is that sending a flow through the shortest path does not produce negative cycles in the residual network. Hence, when all supply is sent to the demand nodes and the feasible solution is achieved, it will be an optimal one.

We also introduce generalized network flows based on Chapter 15 of \cite{ahuja1988network}. In some cases, the flow in a particular edge may be increased or decreased by a multiplier after it leaves the left node of the edge. Denote the multipliers with $m_{i,j}$ for $(i,j) \in E$. The generalized  minimum-cost flow problem is then formulated as
\begin{equation}
\label{eq:gen_min_cost_flow}
\begin{aligned}
&\text{minimize}  && \displaystyle\sum_{(i,j) \in  E} c_{i,j} z_{i, j}, &&\\
&\text{subject to}    &&\displaystyle\sum_{j:(i,j) \in E} m_{i,j}z_{i,j} - \displaystyle\sum_{j:(j,i) \in E} z_{j,i} = b_i,  && i \in G \\
        &    &&0 \leq z_{i,j} \leq l_{i, j}, &&  (i,j) \in E.
\end{aligned}
\end{equation}
If the multiplier is greater than 1, then the flow is increased while if it is smaller than 1, then the flow is decreased.

Different theoretical results hold for the generalized minimum-cost flow problem (\ref{eq:gen_min_cost_flow}). Fortunately, our particular case of (\ref{eq:gen_min_cost_flow}) allows obtaining similar properties as we have in the ordinary minimum-cost flow problem (\ref{eq:min_cost_flow}).

\section{Duality and the combinatorial approach}
\label{sec_app:duality}

In this section, the dual optimization problems of (\ref{eq:T_L_linear_program}) and (\ref{eq:T_P_linear_program}) are considered. In our particular case, the dual problems are interesting since they can be modeled using graph theory and can be solved using combinatorial optimization methods. These combinatorial algorithms may not be more efficient than the simplex method used for solving linear programs, but their development is important since they allow us to prove some interesting properties of the estimated fuzzy set. 
We examine optimization problem (\ref{eq:T_L_linear_program}). First, we eliminate variables $x_u, u \in U$, using constraints $x_u = y_u + \Bar{A}_{\varphi}(u) - \alpha_u$ and we denote $M(u,v) = 1 - \widetilde{R}_{\varphi}(u,v)$. Then, the problem is reformulated as

\begin{equation}
\label{eq:T_L_linear_program2}
\begin{aligned}
&\text{maximize}  && p\displaystyle\sum_{u \in U}\alpha_u - \displaystyle\sum_{u \in U}y_u, &&\\
&\text{subject to}    &&\alpha_v - \alpha_u \leq M(u,v),  && u, v \in U\\
                &    &&\alpha_u - y_u \leq \Bar{A}_{\varphi}(u),  && u\in U\\
  &              && y_u \geq 0 &&u \in U.
\end{aligned}
\end{equation}
Its dual problem is then

\begin{equation}
\label{eq:T_L_linear_dual}
\begin{aligned}
&\text{minimize}  && \displaystyle\sum_{u,v \in U}  M(u,v)z_{u,v} + \displaystyle\sum_{u \in U} \Bar{A}_{\varphi}(u) z_{0, u}  &&\\
&\text{subject to}    &&-z_{0, u} + \displaystyle\sum_{v \in U} z_{u,v} - \displaystyle\sum_{v \in U} z_{v,u} = -p,   &&u \in U \\
        &    &&z_{0, u} \leq 1.   &&u \in U.
\end{aligned}
\end{equation}

In (\ref{eq:T_L_linear_dual}), variables $z_{u,v},\, u,v \in U$, correspond to the first set of constraints from primal (\ref{eq:T_L_linear_program2}), while variables $z_{0, u}, u\in U$, correspond to the second set of constraints from the primal. The first set of constraints in (\ref{eq:T_L_linear_dual}) corresponds to variables $\alpha_u, u\in U$, from the primal, while the second set of constraints corresponds to variables $y_u, u \in U $, from the primal.

If we sum up the equality constraints, we get $\sum_{u \in U} z_{0, u} = np$ where $n = |U|$. Bearing that in mind, we see that (\ref{eq:T_L_linear_dual}) is exactly the minimum-cost flow problem on $n + 1$ nodes where we have one supply node with imbalance $b_0 = np$ and $n$ demand nodes with imbalances $-p$. From the supply node, to all other nodes we have flow $z_{0, u}$, costs $\Bar{A}_{\varphi}(u)$, while all capacities are equal to $1$. Among the demand nodes, there is a flow $z_{u,v}, u,v \in U$, costs $M(u,v)$, and there are no capacity constraints.

To make our model even simpler, we utilize the $T$-transitivity of the relation $\widetilde{R}$. It is easy to verify that the $T$-transitivity is equivalent to $M(u,v) + M(v,w) \geq M(u,w)$ for $u,v,w \in U$. Using this fact, we have that there is an optimal flow which does not use two consecutive edges that are between demand nodes. Assume that for an optimal flow $z^*$ we have $z^*_{u,v} > 0$ and $z^*_{v, w} > 0$, and let $\delta = \min(z^*_{u,v}, z^*_{v, w})$. Then the flow $z^*_{u,v} - \delta, z^*_{v,w}-\delta, z^*_{u,w} + \delta$ is feasible and at most as expensive as the previous flow, i.e., it is optimal. The new flow does not use two consecutive edges since either $z^*_{u,v} - \delta$ or $z^*_{v,w} - \delta$ is 0. The previous elaboration further implies that an optimal flow from the supply node can travel through at most one intermediary node to the destination demand node. Hence, our initial network flow on $n+1$ nodes can be transformed into a flow network on $2n+1$ nodes which has a form of a bipartite graph plus the supply node. One independent set in the bipartite graph is formed by the intermediate nodes, while the other independent set is formed by the destination nodes.  

\begin{figure}[!htb]
\centering
    \begin{tikzpicture}[thick, main/.style = {draw, circle}]
        \tikzstyle{circ}  = [circle, minimum width=15pt, draw, inner sep=0pt]
        
        \node[circ] (1) at (-2,2) {$0$};
        \node[circ] (2) at (0,3.5) {$e_{u_1}$};
        \node[circ] (3) at (0,2) {$e_{u_2}$};
        \node[circ] (4) at (0,0.5) {$e_{u_3}$};
        \node[circ] (5) at (2,3.5) {$f_{u_1}$};
        \node[circ] (6) at (2,2) {$f_{u_2}$};
        \node[circ] (7) at (2,0.5) {$f_{u_3}$};

        \draw[->] (1) -- (2);
        \draw[->] (1) -- (3);
        \draw[->] (1) -- (4);
        \draw[->] (2) -- (5);
        \draw[->] (2) -- (6);
        \draw[->] (2) -- (7);
        \draw[->] (3) -- (5);
        \draw[->] (3) -- (6);
        \draw[->] (3) -- (7);
        \draw[->] (4) -- (5);
        \draw[->] (4) -- (6);
        \draw[->] (4) -- (7);
    \end{tikzpicture}
\caption{Flow modeled as a bipartite graph}
\label{fig:bipartite_network_flow}
\end{figure}
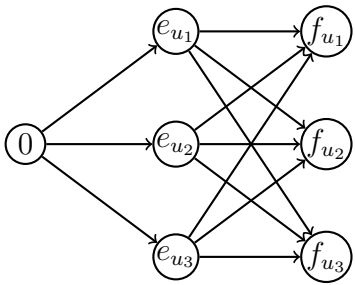

In Figure \ref{fig:bipartite_network_flow}, we have an example of a bipartite network on set of instances $U = \{u_1,u_2,u_3\}$. Since $n=3$ in this case, the bipartite graph has $2\cdot 3 + 1 = 7$ nodes. Node $0$ is the supply node with imbalance $np$. Nodes $\{e_{u_1}, e_{u_2}, e_{u_3}\}$ are the intermediate nodes without imbalances while $\{f_{u_1}, f_{u_2}, f_{u_3}\}$ are the destination nodes with demands $-p$. For $u \in U$, the cost of edges $(0,e_u)$ is $\Bar{A}_{\varphi}(u)$ while the capacity is 1. For $u,v \in U$, the cost of edges $(e_u,f_v)$ is $M(v,u)$ while the capacity is unbounded. The cost of edges $(e_u, f_u)$ is then $0$. If a flow takes path $(0,e_u, f_v)$ in the bipartite graph for $u,v \in U$, and $u \neq v$, then in the original network it means that the flow travels from $0$ to $v$ using intermediate node $u$. If $u=v$, it means that there were no intermediate nodes and that the flow travels directly from $0$ to $u$. 

For a given flow in a bipartite network flow, there is also the corresponding residual network. In such residual network, there are edges from the destination nodes to the intermediate nodes and from the intermediate nodes to the supply. The costs and the capacities of the new edges are then calculated as it was explained in \ref{sec_app:min_cost_flow}.

The bipartite network representation is useful from the perspective of the flow decomposition. For a feasible flow, it is easy to represent it as a sum of simple flows that go from the supply node to the destination node. In the original network, one node can be a destination node for some flow but also an intermediate node for a different flow. Hence, the decomposition is harder in the original network. The decomposition will be important later when dealing with the dual of (\ref{eq:T_P_linear_program}).

The next question is how to reconstruct optimal solution of the primal problem, i.e., to calculate $\alpha^*$ from a solution of the dual $z^*$. Following the duality theory provided in \cite{ahuja1988network}, an optimal vector $\alpha^*$ can be obtained as lengths of shortest paths from the supply node to the corresponding destination nodes in the residual network of $z^*$. 

Now, we examine the dual of (\ref{eq:T_P_linear_program}). The linear program here can be rewritten similarly as (\ref{eq:T_L_linear_program2}), just with the different granularity constraints. Here instead of $\alpha_v - \alpha_u \leq M(u,v)$ we have $\alpha_v\widetilde{R}_{\varphi}(u,v) \leq \alpha_u$. The dual of such formulated problem is then

\begin{equation}
\label{eq:T_P_linear_dual}
\begin{aligned}
&\text{minimize}   &&   \sum_{u \in U} \Bar{A}_{\varphi}(u) z_{0, u},  &&\\
&\text{subject to}    &&-z_{0, u} + \widetilde{R}_{\varphi}(u,v)\sum_{v \in U} z_{u,v} - \sum_{v \in U} z_{v,u} = -p,  && u \in U \\
        &    &&z_{0, u} \leq 1.   &&u \in U
\end{aligned}
\end{equation}

The difference between (\ref{eq:T_L_linear_dual}) and (\ref{eq:T_P_linear_dual}) is that in the latter, we have multipliers $\widetilde{R}_{\varphi}(u,v), u,v \in U$, instead of costs on the edges. Due to the multipliers, we now deal with the minimum-cost flow problem on a generalized flow network with $n + 1$ nodes among which there are $n$ demand nodes with demand $-p$ and one supply node with an unspecified amount of supply.

We may notice that in this case the edges of the network consist of two different groups. The first group is formed by the edges from the supply nodes to the demand nodes. Those edges have costs and do not have multipliers. The second group is formed by the edges among the demand nodes. Those edges, conversely, have multipliers and do not have costs. Similarly to (\ref{eq:T_L_linear_dual}), we are able to utilize the $T$-transitivity of $\widetilde{R}_{\varphi}$ w.r.t. $T_P$ in a way that there is an optimal flow which does not use two consecutive edges from the second group. If we have three demand nodes $u,v,w \in U$ in a network and an optimal flow that uses edges $(u,v)$ and $(v,w)$, we can redirect the flow to use only edge $(u,w)$ and the redirected flow will have smaller or equal loss than the original flow. That will further lead to the smaller or equal cost of the redirected flow which makes it optimal. Therefore, as above, there is an optimal solution in which a flow travels from the supply node to the destination demand node using at most one intermediate node. That again further implies that the initial general network on $n+1$ nodes can be transformed into a generalized bipartite flow network on $2n+1$ nodes. For the new network, the same model applies as in Figure \ref{fig:bipartite_network_flow}. Using this model, we can clearly see the difference between two groups of edges introduced above. The first group is formed by the edges between the supply node and the left partition of the bipartite graph (intermediate nodes), while the second group is formed by the edges between the two partitions of the bipartite graph.

As before, for a given flow on the generalized bipartite network, we have the corresponding residual network. The same properties apply as above except the case when the flow passes through an edge with multiplier. In that case, if the original edge has multiplier $\widetilde{R}_{\varphi}(u,v)$ then the reverse edge in the residual network will have multiplier $\frac{1}{\widetilde{R}_{\varphi}(u,v)}$ which is an edge of a gain type (greater than 1).

We will now construct a new algorithm for solving a generalized minimum-cost flow problem on a generalized bipartite flow network. The algorithm is based on the existing SSP algorithm presented in Algorithm \ref{algo:ssp}. Assume that we have a demand node $f_u$ to which we want to deliver some flow $b$. We want to deliver the flow at the cheapest possible price. If we deliver a flow using intermediate node $e_v$, then the amount of flow that we have to take from the supply node is $\frac{b}{\widetilde{R}_{\varphi}(v,u)}$ and the cost of such flow is $\frac{b\Bar{A}_{\varphi}(v)}{\widetilde{R}_{\varphi}(v,u)}$. In general, a price to deliver a unit of flow is a ratio of the cost of an edge from the supply to the first partition and the product of multipliers of edges that connect the two partitions. Bear in mind that in the residual network, a flow may use multiple edges between partitions (edges with multipliers) to deliver the flow. Using this, we construct the greedy approach presented as Algorithm \ref{algo:gssp}.

\begin{algorithm}[H]
    \caption{Generalized successive shortest paths}
\begin{algorithmic}[1]
\State \textbf{Input:} Bipartite flow network $F$.
\State \textbf{Output:} Flow $z$.
\State{Set initial flow $z_{i,j} = 0$, $(i,j) \in E$}
\State{Set initial residual network $F' = F$}
\While {there exists demand value different than 0}
    \State{Pick a demand node $i$}
    \State{Calculate the smallest possible cost from the supply node to $i$}
    \State{Calculate the largest amount of flow that can be sent through the least costly path}
    \State{Send the calculated flow through the least costly path}
    \State{Update $F'$}
\EndWhile
\State{Reconstruct $z'$ from $F'$}
\end{algorithmic}
\label{algo:gssp}
\end{algorithm}
To calculate the smallest possible cost from the supply node, we can use a shortest path method. We want to minimize the ratio of one cost value (from the supply to the intermediate nodes) and a product of multipliers (between intermediate and destination nodes). If we apply logarithms on the cost values and reciprocals of the multipliers, we may apply the Bellman-Ford algorithm to calculate the shortest path between the supply node and the chosen demand node in order to obtain a least costly way to transport the flow.

After the shortest path is determined, we have to calculate the amount of flow that will be taken from the supply node in order to deliver the maximal amount of flow to the demand node. In comparison with the standard minimum-cost flow problem, here we have to take into account all the loses and gains that happen during the flow transfer. Denote the shortest path in the residual network with $P = (0, e_{u_1}, f_{u_2}, e_{u_3}, \dots, f_{u_k})$ and let $b$ be a demand of node $f_{u_k}$. We would like to deliver $|b|$ ($|\cdot|$ stands for absolute value) amount of flow to the demand node from the supply node, but that is not always possible due to the capacities of particular edges on path $P$. The maximal amount of flow can be determined recursively. The maximal amount of flow that can be transferred from node $f_{u_{k-2}}$ to node $f_{u_{k}}$ is bounded with the capacity of the reverse edge $l'_{f_{u_{k-2}},e_{u_{k-1}}}$ and the demand divided with the loses on the edges in between $\frac{|b|\widetilde{R}_{\varphi}(u_{k-1},u_{k-2})}{\widetilde{R}_{\varphi}(u_{k-1}, u_k)}$. Using that reasoning, if we set the initial value $z' = |b|$, then we can use the following iteration formula.
$$
z' = \min \left( \frac{z'\widetilde{R}_{\varphi}(u_{k-2i+1},u_{k-2i})}{\widetilde{R}_{\varphi}(u_{k-2i+1}, u_{k-2i+2})}, l'_{f_{u_{k-2i}},e_{u_{k-2i+1}}} \right),
$$
for $i$ going from $1$ to $\frac{k}{2} - 1$. The last step is $z' = \min (\frac{z'}{\widetilde{R}_{\varphi}(u_1, u_2)}, l'_{0, e_{u_1}})$ for subpath $(0, e_{u_1}, f_{u_2})$.

After $z'$ is calculated, we have to determine the amount of flow that will end up in the demand node $f_{u_k}$ as well as to update the residual network on path $P$. In the first step, $z'$ leaves the supply node, passes node $e_{u_1}$ and enters node $f_{u_2}$. On edge $(e_{u_1}, f_{u_2})$ it was multiplied with $\widetilde{R}_{\varphi}(u_1, u_2)$: $z' = \widetilde{R}_{\varphi}(u_1, u_2)z'$. Then we update the residual network on edges $(f_{u_2}, e_{u_1})$ and $(f_{u_2}, e_{u_3})$: $l'_{f_{u_2}, e_{u_1}} = l'_{f_{u_2}, e_{u_1}} + z'$, $l'_{f_{u_2}, e_{u_3}} = l'_{f_{u_2}, e_{u_3}} - z'$ and we send the flow to the next node from the second partition and repeat the process. After the remaining flow arrives to the demand node, we increase the imbalance of the demand node.

Since Algorithm \ref{algo:gssp} is novel, we do not benefit from the existing theory as we did in case of Algorithm \ref{algo:ssp}. In \ref{sec_app:proof_of_correctness}, we showed that Algorithm \ref{algo:gssp} indeed returns an optimal result, as well as how to construct a solution of the primal problem from the solution of the dual one. As is shown in \ref{sec_app:proof_of_correctness}, $\alpha^*$ is constructed by performing step 7 (without logarithms) of Algorithm \ref{algo:gssp} on the residual network of $z^*$, i.e., it is the smallest possible cost of the transport from the supply node to the destination nodes.


\section{Proof of correctness for Algorithm \ref{algo:gssp}}
\label{sec_app:proof_of_correctness}

In this section we prove that Algorithm \ref{algo:gssp} terminates and that it outputs an optimal solution. Also, we construct a way to obtain a solution of the primal problem from the solution of the dual one.

We first prove the termination.

\begin{proposition}
Assume that all parameters in Algorithm \ref{algo:gssp} are rational numbers. Then algorithm \ref{algo:gssp} terminates.
\end{proposition}

\begin{proof}
It is easy to see that if we multiply the right side of the constraints in (\ref{eq:T_P_linear_dual}) with a positive constant $C$, the the optimal solution is $Cz^*$ where $z^*$ is the solution of the initial problem. For some parameter $a$ in (\ref{eq:T_P_linear_dual}) we have its rational representation $a = \frac{q}{r}$ for $q$ and $r$ being integers. Let $C$ be the least common multiple (LCM) of all integers $q$ and $r$ for all parameters in (\ref{eq:T_P_linear_dual}). If we multiply the right side of the constraints in (\ref{eq:T_P_linear_dual}) with $C$, then all the demand value will become integers and all intermediate flows in Algorithm \ref{algo:gssp} will become integers. That further implies that all the updates on demands in Algorithm \ref{algo:gssp} will be integer which further implies that the algorithm will terminate in at most $Cpn$ steps.
\end{proof}

In practice, the termination is always guaranteed since computers can work only with rational numbers

Now, lets define a flow cycle in the residual generalized bipartite network. The cycle starts with an edge from the first part (costly edges without multipliers) of the network, then it contains edges from the second part (edges with multipliers without costs) and ends with a reverse edge from the first part. A model of such cycle is shown in Figure \ref{fig:cycle_generalized}.

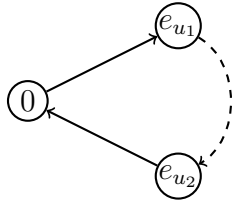
\begin{figure}[!htb]
\centering
    \begin{tikzpicture}[thick, main/.style = {draw, circle}]
        \tikzstyle{circ}  = [circle, minimum width=15pt, draw, inner sep=0pt]
        
        \node[circ] (1) at (-2,1) {$0$};
        \node[circ] (2) at (0,2) {$e_{u_1}$};
        \node[circ] (3) at (0,0) {$e_{u_2}$};

        \draw[->] (1) -- (2);
        \draw[->] (3) -- (1);
        \draw[->] (2) to [bend left=60] (3)[dashed];
    \end{tikzpicture}
\caption{Cycle in a generalized bipartite network}
\label{fig:cycle_generalized}
\end{figure}

In Figure \ref{fig:cycle_generalized}, the dashed line between $e_{u_1}$ and $e_{u_2}$ stands for the subpath that contains only the edges from the second part of the residual network. Also, it may hold that $e_{u_1} \equiv e_{u_2}$. In that case, the cycle consists only of the edges from the second part. Let $\widetilde{R}_{\varphi} (e_{u_1}, e_{u_2})$ be a multiplier of the path that consists of the edges from the second part of the residual network, i.e., a product of the multipliers on the edges from the path. We say that the cycle is of negative cost if $A_{\varphi}(u_1) < \widetilde{R}_{\varphi} (e_{u_1}, e_{u_2}) A_{\varphi}(u_2)$. For the reminder, $A_{\varphi}(u_1)$ and $A_{\varphi}(u_2)$ are the costs on edges $(0, e_{u_1})$ and $(0, e_{u_2})$. The reason why the cycle is of negative cost is that if we send a unit of flow along it, the cost of that flow is $A_{\varphi}(u_1) - \widetilde{R}_{\varphi} (e_{u_1}, e_{u_2}) A_{\varphi}(u_2)$, i.e., the cost is negative. Such flow would not change any demand value on the destination nodes but it will reduce the overall cost of the flow.

The next proposition utilizes the bipartite representation of the flow network.

\begin{proposition}
\label{prop:decomposition}
Every flow in a generalized bipartite network can be represented as a sum of a finite number of simple path flows from the supply node to a destination node.
\end{proposition}
\begin{proof}
Let $z$ be a flow and consider an edge $(e_{u_1}, f_{u_2})$ from the second part of the network and its flow value $z_{e_{u_1}, f_{u_2}}$. That edge receives a flow from edge $(0, e_{u_1})$ which is a part of path flow $z_P$ from path $P = (0, e_{u_1}, f_{u_2})$ that connects the supply node and the destination node $f_{u_2}$. $z_P$ is then a summand in the representation while the remaining flow $z-z_P$ has no flow on the edge $(e_{u_1}, f_{u_2})$ and hence we can remove that edge from the network flow. If we continue, in every step we will construct one summand and remove one edge from the second part of the network. Since we have a finite number of edges, we have a finite number of summands.
\end{proof}

We have the following result.

\begin{proposition}
\label{prop:neg_cost_cycle_optimal}
Solution $z^*$ is optimal in the generalized bipartite network if and only if its residual network does not contain negative cost cycles.
\end{proposition}

\begin{proof}
$(\Rightarrow)$ When the solution is optimal, there are no negative cost cycles. If otherwise, we could send a flow through a negative cost cycle and we would decrease the cost of the overall flow as described above. That contradicts the optimality. 

$(\Leftarrow)$ Assume that $z^*$ is a feasible solution whose residual network does not contain negative cost cycles and let $z'$ be a feasible solution. Let $z' = z^* + z''$. We first show that $z''$ is a feasible flow from the residual network of $z^*$, i.e., it satisfies its constraints. For an edge $(0, e_{u_1})$ if the flows are different, we can have either $z'_{0,e_{u_1}} > z^*_{0,e_{u_1}}$ or $z'_{0,e_{u_1}} < z^*_{0,e_{u_1}}$. In the first case, it holds that $z'_{0,e_{u_1}} = z^*_{0,e_{u_1}} + z''_{0,e_{u_1}}$, i.e., $z''_{0,e_{u_1}}$ uses the original edge. Since $z'_{0,e_{u_1}} \leq 1$ then $z''_{0,e_{u_1}} \leq 1 - z^*_{0,e_{u_1}}$ which is a constraint from the residual network. In the second case, it holds that $z'_{0,e_{u_1}} = z^*_{0,e_{u_1}} - z''_{e_{u_1}, 0}$, i.e., $z''_{e_{u_1}, 0}$ uses the reverse edge. Since $z'_{0,e_{u_1}} \geq 0$ then $z''_{e_{u_1}, 0} \leq z^*_{0,e_{u_1}}$ which is a constraint for the reverse edge from the residual network. Using similar reasoning, we can conclude the same for the whole network. 

The next step is to show that $z''$ is a sum of a finite number of simple flow cycles, as shown in Figure \ref{fig:cycle_generalized}, i.e., it has a cycle representation. Proposition \ref{prop:decomposition} states that both flows $z'$ and $z^*$ are sums of simple flows on paths from the supply node to a destination node. Take a summand $z'_{P_1}$ of $z'$ and summand $z^*_{P_2}$ of $z^*$ for $P_1 = (0, e_{u_1}, f_{u_3})$ and $P_2 = (0, e_{u_2}, f_{u_3})$. The paths have the same destination node. Assume that the first summand delivers amount $b_1$ of flow to the destination node while the second delivers amount $b_2$ of flow to the same node. W.L.O.G. assume that $b_1 \geq b_2$. Then the flow $\frac{b_2}{b_1}z'_{P_1} - z^*_{P_2}$ is a flow along cycle $(0, e_{u_1}, f_{u_3}, e_{u_2},0)$ and one of the summands in the cycle representation of $z''$. After the summand is identified, we remove its flow from the consideration. In that moment, $z^*_{P_2}$ is fully removed while we are left with $(1-\frac{b_2}{b_1})z'_{P_1}$ from the first path. We continue to create flow cycles as summands from the remaining path flows from $z'$ and $z^*$. Since after every summand is identified we remove one path flow, the number of summands is finite. Hence, $z''$ is a sum of a finite number of cycle flows. Since $z''$ is a flow in the residual network of $z^*$, all the cycles from its cycle representation are of positive cost by the assumption which implies that $z''$ is of positive cost. Since the cost of $z'$ is a sum of costs of $z^*$ and $z''$, cost of $z'$ is larger than the cost of $z^*$. Since flow $z'$ was an arbitrary feasible flow, we conclude that $z^*$ is an optimal flow.
\end{proof}

\begin{proposition}
Algorithm \ref{algo:gssp} returns an optimal solution.
\end{proposition}
\begin{proof}
Assume that in one iteration of Algorithm (\ref{algo:gssp}), the shortest path had the form $P_1 = (0, e_{u_2}, \dots, f_{u_3})$ and that after the step, the negative cost cycle $(0, e_{u_1}, \dots, f_{u_3}, \dots, e_{u_2}, 0)$ was formed. The negative cost cycle is formed from the path $P_2 = (0, e_{u_1}, \dots, f_{u_3})$ and the reverse path $P_1$. The model of such cycle is represented in Figure \ref{fig:one_step_gssp}.
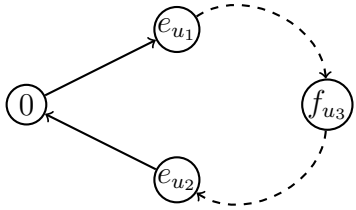
\begin{figure}[H]
\centering
    \begin{tikzpicture}[thick, main/.style = {draw, circle}]
        \tikzstyle{circ}  = [circle, minimum width=15pt, draw, inner sep=0pt]
        
        \node[circ] (1) at (-2,1) {$0$};
        \node[circ] (2) at (0,2) {$e_{u_1}$};
        \node[circ] (3) at (2,1) {$f_{u_3}$};
        \node[circ] (4) at (0,0) {$e_{u_2}$};

        \draw[->] (1) -- (2);
        \draw[->] (4) -- (1);
        \draw[->] (2) to [bend left=60] (3)[dashed];
        \draw[->] (3) to [bend left=60] (4)[dashed];
    \end{tikzpicture}
\caption{Cycle after one step of Algorithm \ref{algo:gssp}}
\label{fig:one_step_gssp}
\end{figure}
The dots in the cycle as well as dashed edges in the figure stand for edges from the second part of the residual network (edges with multipliers). If the cycle is negative, then $A_{\varphi}(u_1) < \widetilde{R}_{\varphi} (e_{u_1}, f_{u_3})\widetilde{R}_{\varphi} (f_{u_3}, e_{u_2}) A_{\varphi}(u_2)$. The latter is equivalent to $  \frac{A_{\varphi}(u_1)}{\widetilde{R}_{\varphi} (e_{u_1}, f_{u_3})} < \frac{A_{\varphi}(u_2)}{\widetilde{R}_{\varphi} (e_{u_2}, f_{u_3})}$ which states that path $P_2$ is actually shorter than $P_1$ which contradicts the assumption that $P_1$ is the shortest path at this step.

Hence, at every iteration of Algorithm \ref{algo:gssp}, there are no negative cost cycles and as soon as the feasible solution is achieved, it will be an optimal one according to Proposition \ref{prop:neg_cost_cycle_optimal}.
\end{proof}

After we constructed the algorithm that solves the dual optimization problem, we need to obtain an optimal solution for the primal which was our initial goal. First, we need one technical proposition.

\begin{proposition}
\label{prop:middle_edge}
For a given generalized bipartite network, there exists optimal solution $z^*$ for which it holds 
$$z^*_{0, e_{u}} > 0 \implies z^*_{e_u, f_u} > 0. $$ 
\end{proposition}
\begin{proof}
Assume that for some solution $z^*$ and some instance $u$ we have that $z^*_{0, e_{u}} > 0$ and $z^*_{e_u, f_u} = 0$. Then, in the simple path decomposition of the flow, we have path $(0, e_{v}, f_{u})$ that delivers flow to $f_u$, and path $(0, e_u, f_w)$ that uses flow from edge $(0, e_u)$. Then, in the residual network of $z^*$, $C = (e_u, f_u, e_v, f_w, e_u)$ is a cycle. Due to transitivity of $\widetilde{R}$, it holds 
$$
\widetilde{R}_{\varphi}(v,u)\widetilde{R}_{\varphi}(u,w) \leq \widetilde{R}_{\varphi}(v,w).
$$
If $\widetilde{R}_{\varphi}(v,u)\widetilde{R}_{\varphi}(u,w) < \widetilde{R}_{\varphi}(v,w)$, then $C$ is a negative cost cycle which contradicts the optimality of $z^*$. If $R_{\varphi}(v,u)R_{\varphi}(u,w) = R_{\varphi}(v,w)$ then cycle $C$ is a zero-cost cycle and a flow can be sent through the cycle without violating optimlaity. Hence, sending some amount of flow through the cycle, we will construct another optimal solution $z^{**}$ where $z^{**}_{e_u, f_u} > 0$.
\end{proof}
In practice, if we obtain an optimal solution containing an edge for which the previous proposition does not hold, we can get another optimal solution, without such edges, as explained in the proof of the previous proposition. From now on, we assume that we have an optimal solution for which the previous proposition holds. 

We continue with the duality theory of the linear programs. 

According to the strong duality theorem \cite{matousek2007understanding}, if there exists an optimal solution of the dual problem $z^*$ then, there exists an optimal solution for the primal problem $\alpha^{*}$, and it holds that the values of objectives in (\ref{eq:T_L_linear_program2}) and in (\ref{eq:T_P_linear_dual}) are equal, i.e.,
\begin{equation}
\label{eq:diality_expression}
\sum_{u \in U} \Bar{A}_{\varphi}(u) z^*_{0, u} = \displaystyle\sum_{u \in U}  p\alpha^*_u - \sum_{u \in U} \max( \alpha^*_u - \Bar{A}_{\varphi}(u), 0).
\end{equation}
In the previous expression, $y_u$ is replaced with its definition. In an optimal solution, we have that
\begin{equation}
\label{eq:nodes_balance}
    \sum_{v \in U} z_{u, v} = z_{0, u}, \quad \sum_{u \in U} \widetilde{R}_{\varphi} (u,v) z_{u,v} = p.
\end{equation}
We have the following equalities
\begin{align*}
    \sum_{u \in U} \max( \alpha^*_u - \Bar{A}_{\varphi}(u), 0) &= \displaystyle\sum_{u \in U}  p\alpha^*_u - \sum_{u \in U} \Bar{A}_{\varphi}(u) z^*_{0, u} \\
    &= \displaystyle\sum_{u \in U}  p\alpha^*_u - \sum_{u \in U} (\Bar{A}_{\varphi}(u) - \alpha^* _u )z^*_{0, u} - \sum_{u \in U} \alpha^* _u z^*_{0, u} \\
    &= \displaystyle\sum_{u \in U}  p\alpha^*_u - \sum_{u \in U} (\Bar{A}_{\varphi}(u) - \alpha^* _u )z^*_{0, u} - \sum_{u \in U} \alpha^* _u \sum_{v \in U} z^*_{u, v} \\
    &= \displaystyle\sum_{u \in U}  p\alpha^*_u - \sum_{u \in U} (\Bar{A}_{\varphi}(u) - \alpha^* _u )z^*_{0, u} \\ & - \sum_{u,v \in U} (\alpha^* _u - \widetilde{R}_{\varphi}(u, v) \alpha^* _v)  z^*_{u, v}
    - \sum_{u,v \in U}  \widetilde{R}_{\varphi}(u, v) \alpha^* _v  z^*_{u, v} \\
    &= \displaystyle\sum_{u \in U}  p\alpha^*_u - \sum_{u \in U} (\Bar{A}_{\varphi}(u) - \alpha^* _u )z^*_{0, u} \\ & - \sum_{u,v \in U} (\alpha^* _u - \widetilde{R}_{\varphi}(u, v) \alpha^* _v)  z^*_{u, v}
    - \sum_{v \in U} \alpha^* _v   \sum_{u \in U} \widetilde{R}_{\varphi}(u, v)  z^*_{u, v} \\
    &=  \sum_{u \in U} (\alpha^* _u - \Bar{A}_{\varphi}(u))z^*_{0, u}  - \sum_{u,v \in U} (\alpha^* _u - \widetilde{R}_{\varphi}(u, v) \alpha^* _v)  z^*_{u, v}
\end{align*}

The second equality holds because of the left expression in (\ref{eq:nodes_balance}) while the last equality holds because the right expression in (\ref{eq:nodes_balance}). We have that for all $u \in U$, $\max( \alpha^*_u - \Bar{A}_{\varphi}(u), 0) \geq (\alpha^* _u - \Bar{A}_{\varphi}(u))z^*_{0, u}$ and that for all $u,v \in U$, $\alpha^* _u - \widetilde{R}_{\varphi}(u, v) \alpha^* _v \geq 0$, since $\alpha^*$ is a feasible solution. Hence, for the previous equality to hold, we need to have that for all $u \in U $, $\max( \alpha^*_u - \Bar{A}_{\varphi}(u), 0) = (\alpha^* _u - \Bar{A}_{\varphi}(u))z^*_{0, u}$ and that for all $u,v \in U$, $(\alpha^* _u - \widetilde{R}_{\varphi}(u, v) \alpha^* _v)  z^*_{u, v} = 0$. The latter is equivalent to the following set of conditions. 
\begin{itemize}
    \item $z^*_{0, u} = 0 \implies \alpha^*_{u} \leq \Bar{A}_{\varphi}(u)$,
    \item $0 < z^*_{0, u} < 1 \implies \alpha^*_{u} = \Bar{A}_{\varphi}(u)$,
    \item $z^*_{0, u} = 1 \implies \alpha^*_{u} \geq \Bar{A}_{\varphi}(u)$,
    \item $z^*_{u,v} > 0 \implies \alpha^* _u - \widetilde{R}_{\varphi}(u, v) \alpha^* _v = 0$,
\end{itemize}
for $u,v \in U$. We have the following conclusion: if we solve the dual optimization problem and obtain an optimal solution $z^*$, then a solution of the primal optimization problem is any $\alpha^*$ which satisfies the conditions listed above.

Moreover, $\alpha^*$ can be constructed by performing step 7 of Algorithm \ref{algo:gssp} on the residual network of $z^*$, i.e., it is the smallest possible cost of the transport from the supply node to the destination nodes. It is easily verifiable that such $\alpha^*$ satisfies the conditions above. The proof of this verification lies in that if we assume that some condition is not satisfied, then we would have a negative cost cycle which contradicts the optimality of $z^*$. To prove the contradiction, we need Proposition \ref{prop:middle_edge}.

\section{Proof of Proposition \ref{prop:parameter_monotonicity}}
\label{sec_app:proof_parameter_monotonicity}
Let $\alpha^p_{u} = \varphi(\hat{A}_p(u))$ and $\alpha^q_{u} = \varphi(\hat{A}_q (u))$ for $u \in U$. Then $\hat{A}_p(u) \leq \hat{A}_q(u) \Leftrightarrow \alpha^p_{u} \leq \alpha^q_{u}$. To prove this proposition, we will use Algorithm \ref{algo:ssp} in case of $T_L$ and Algorithm \ref{algo:gssp} in case if $T_P$. We apply both algorithms on the bipartite flow network in the way that we first deliver amount $p$ of flow to every destination node, then we calculate $\alpha^p$ as the smallest cost from the supply node to the destination nodes in the residual network, then we deliver additional amount $q - p$ of flow to every destination node and then we calculate $\alpha^q$ in the same way as $\alpha^p$. Using this procedure, we may notice that to calculate $\alpha^q$ we need a few more iterations of the algorithms after $\alpha^p$. Bearing this in mind, it is enough to prove that after every iteration of the algorithm, i.e., after sending some amount of flow to a destination node and updating the residual network, the cost from the supply node to every destination node stayed the same or is increased.

When updating residual network $F'$, the possible changes in the residual networks are the following:
\begin{itemize}
    \item Reverse edges between the supply node and the intermediate nodes can be added while the original edges can be removed.
    \item Reverse edges between the intermediate and destination nodes can be added or removed.
\end{itemize}
Adding reverse edges between the supply node and intermediate nodes is not important in this case, since shortest paths do not use these edges. Removing the original edges between the same nodes will not reduce the costs since the shortest paths now chose among the smaller set of edges. The same holds if we remove reverse edges between the intermediate nodes. 

The last step is to prove that adding reverse edges between the intermediate and destination nodes will not reduce the costs from the supply to the destination nodes.

For that purpose, we consider Figure \ref{fig:monotonic_costs}.

\begin{figure}[!htb]
\centering
    \begin{tikzpicture}[thick, main/.style = {draw, circle}]
        \tikzstyle{circ}  = [circle, minimum width=15pt, draw, inner sep=0pt]
        
        \node[circ] (1) at (-2.5,2) {$0$};
        \node[circ] (2) at (0,4) {$e_{u_1}$};
        \node[circ] (3) at (2.5,4) {$f_{u_1}$};
        \node[circ] (4) at (2.5,2) {$f_{u_2}$};
        \node[circ] (5) at (2.5,0) {$f_{u_3}$};

        \draw[->] (1) -- node[right] {$C_c$}(2)[dashed];
        \draw[->] (1) to [bend right=30] node[above] {$C_a$}(3)[dashed];
        \draw[->] (1) -- node[above] {$C_y$} (5)[dashed];
        \draw[->] (2) to [bend left=30] node[above] {$x$} (3);
        \draw[->] (3) to [bend left=30] node[above] {$-x$} (2);
        \draw[->] (2) -- node[right] {$C_b$}(5)[dashed];
        \draw[->] (3) -- node[right] {$C_d$}(4)[dashed];
    \end{tikzpicture}
\caption{Flow modeled as a bipartite graph}
\label{fig:monotonic_costs}
\end{figure}
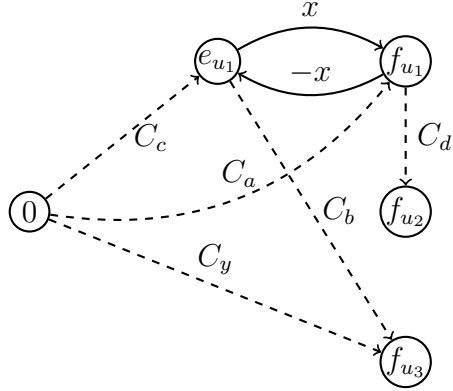

With dashed lines, we denote certain paths whose costs are marked on the figure. In both cases of $T_L$ and $T_P$, the costs are the values used to calculate the shortest paths. Assume that in step $i$, we were calculating the shortest path between $0$ and $f_{u_2}$ and we obtained that the shortest path is $(0, \dots, e_{u_1}, f_{u_1}, \dots, f_{u_2})$ and since some flow is sent through that path, a reverse edge $(f_{u_1}, e_{u_1})$ is created with cost $-x$. Assume that before step $i$, the shortest path from $0$ to $f_{u_3}$ was  $(0, \dots, f_{u_3})$ with cost $C_y$ while after the previous step and after adding reverse edge $(f_{u_1}, e_{u_1})$ the shortest path is $(0,\dots ,f_{u_1}, e_{u_1}, \dots, f_{u_3})$ with cost $C_a - x + C_b$. Then, we have that $C_a + C_b < x+ C_y$. Since the shortest path in step $i$ was $(0, \dots, e_{u_1}, f_{u_1}, \dots, f_{u_2})$, it holds that $ C_c + x \leq C_a$. Adding this to the previous expression, we have that 
$$
x + C_y > C_a + C_b \geq C_c + x + C_b \Leftrightarrow C_y > C_c + C_b.
$$
The last inequality contradicts the assumption that before step $i$, the smallest cost between $0$ and $f_{u_3}$ is $C_y$.

\end{document}